\documentclass{article}

\PassOptionsToPackage{sort,numbers}{natbib}
\usepackage[final]{nips_2017}

\usepackage{amsmath, amssymb, amsthm}
\usepackage{graphicx}
\usepackage{times}
\usepackage{subfigure} 
\usepackage{array}
\usepackage{caption} 
\usepackage{tabularx}
\newcolumntype{Y}{>{\tiny\raggedleft\arraybackslash}X}
\usepackage{algorithm, algorithmic}
\usepackage{color}\definecolor{darkblue}{rgb}{0,0,0.75}
\usepackage[pdftitle={Effective Parallelisation for Machine Learning},colorlinks=true, bookmarksnumbered, citecolor=darkblue, urlcolor=darkblue, linkcolor=darkblue, 
 pdfpagemode=UseNone]{hyperref}
\usepackage[utf8]{inputenc}
\usepackage{xfrac}
\usepackage{wrapfig}

\newtheorem{thm}{Theorem}
\newtheorem{thmrestate}{Theorem}
\newtheorem{prop}[thm]{Proposition}
\newtheorem{lm}[thm]{Lemma}

\newtheorem{defin}[thm]{Definition}

\DeclareMathOperator{\ld}{ld}

\newcommand{\NC}{NC}
\newcommand{\R}{\mathbb{R}}
\newcommand{\N}{\mathbb{N}}
\newcommand{\dist}{\mathcal{D}}
\newcommand{\defemph}[1]{\emph{#1}}

\newcommand{\convhull}[1]{\langle #1 \rangle}

\newcommand{\mapping}[3]{#1\!: #2 \to #3}
\newcommand{\prob}{P}
\newcommand{\cores}{c}
\DeclareMathOperator*{\Exp}{\mathbb{E}}
\DeclareMathOperator*{\argmin}{\mathrm{argmin}}

\newcommand{\runtime}{T}
\newcommand{\algo}{\mathcal{L}}
\newcommand{\radonAlgo}{\mathcal{R}}

\newcommand{\confBase}{\delta}
\newcommand{\conf}{\Delta}
\newcommand{\loss}{\ell}
\newcommand{\bigo}{\mathcal{O}}
\newcommand{\dataset}{D}

\newcommand{\solution}{f}

\newcommand{\solutionset}{\mathcal{F}}

\newcommand{\samplesize}{N}
\newcommand{\basesamplesize}{n}
\newcommand{\radonPoint}{\mathfrak{r}}
\newcommand{\IRP}{\radonPoint_h} 
\newcommand{\eps}{\varepsilon}
\newcommand{\quality}{\mathcal{Q}}
\newcommand{\qualityOf}[1]{\quality\left( #1 \right)}

\newcommand{\inputspace}{\mathcal{X}}
\newcommand{\inputvar}{x}
\newcommand{\outputspace}{\mathcal{Y}}
\newcommand{\outputvar}{y}

\newcommand{\parallelSchemeName}{Radon machine}

\newcommand{\Comment}[1]{\ \emph{\# #1}}

\newcommand{\ueq}[1][]{%
	\if\relax\detokenize{#1}\relax
	\sbox0{$\underbrace{=}_{}$}%
	\mathrel{\mathmakebox[\wd0]{=}}
	\else
	\mathrel{\underbrace{=}_{\mathclap{#1}}}
	\fi}

\title{Effective Parallelisation for Machine Learning}

\author{
	Michael Kamp\\
	University of Bonn\\
	and	Fraunhofer IAIS\\
	\texttt{kamp@cs.uni-bonn.de} \\
	\And
	Mario Boley\\
	Max Planck Institute for Informatics\\ and Saarland University\\
	\texttt{mboley@mpi-inf.mpg.de} \\
	\And
	Olana Missura\\
	Google Inc.\\
	\texttt{olanam@google.com} \\
	\And
	Thomas G\"artner\\
	University of Nottingham\\
	\texttt{thomas.gaertner@nottingham.ac.uk} \\
}

\begin{document}

\maketitle
\begin{abstract}
We present a novel parallelisation scheme that simplifies the adaptation of learning algorithms to growing amounts of data as well as growing needs for accurate and confident predictions in critical applications.
In contrast to other parallelisation techniques, it can be applied to a broad class of learning algorithms without further mathematical derivations and without writing dedicated code, while at the same time maintaining theoretical performance guarantees.
Moreover, our parallelisation scheme is able to reduce the runtime of many learning algorithms to polylogarithmic time on quasi-polynomially many processing units.
This is a significant step towards a general answer to an open question 
on the efficient parallelisation of machine learning algorithms in the sense of Nick's Class (\NC).
The cost of this parallelisation is in the form of a larger sample complexity. Our empirical study confirms the potential of our parallelisation scheme with fixed numbers of processors and instances in realistic application scenarios.	
\end{abstract}

\section{Introduction}
This paper contributes a novel and provably effective parallelisation scheme for a broad class of learning algorithms.
The significance of this result is to allow the confident application of machine learning algorithms with growing amounts of data. 
In critical application scenarios, i.e., when errors have almost prohibitively high cost, this confidence is essential~\citep{nouretdinov2011machine, sommer2010outside}.
%
%
To this end, we consider the parallelisation of an algorithm to be effective if it achieves the same confidence and error bounds as the sequential execution of that algorithm in much shorter time.
Indeed, our parallelisation scheme can reduce the runtime of learning algorithms from polynomial to polylogarithmic. For that, it consumes more data and is executed on a quasi-polynomial number of processing units.
%


To formally describe and analyse our parallelisation scheme, we consider the regularised risk minimisation setting. For a fixed but unknown joint probability distribution $\dist$ over an \defemph{input space} $\inputspace$ and an \defemph{output space} $\outputspace$, a dataset 
$\dataset\subseteq\inputspace\times\outputspace$ of size $\samplesize\in\N$ drawn iid~from $\dist$, a convex \defemph{hypothesis space} $\solutionset$ of functions $\mapping{\solution}{\inputspace}{\outputspace}$, a loss function $\mapping{\loss}{\solutionset\times\inputspace\times\outputspace}{\R}$ that is convex in $\solutionset$, and a convex regularisation term $\mapping{\Omega}{\solutionset}{\R}$, \defemph{regularised risk minimisation algorithms} solve
\begin{equation}\label{eq:ml}
\algo(\dataset)=\argmin_{\solution\in\solutionset}\sum_{(\inputvar,\outputvar)\in\dataset}\loss\left(\solution,\inputvar,\outputvar\right) + \Omega(\solution)\enspace .
\end{equation}
The aim of this approach is to obtain a hypothesis $\solution\in\solutionset$ with small \defemph{regret}
\begin{equation}\label{eq:q}
\qualityOf{\solution} = 
\Exp \left[\loss\left(\solution,\inputvar,\outputvar\right)\right] - \argmin_{\solution'\in\solutionset} \Exp \left[ \loss\left(\solution',\inputvar,\outputvar\right) \right]\enspace .
\end{equation}

Regularised risk minimisation algorithms are typically
designed to be \defemph{consistent} and \defemph{efficient}.
They are consistent if there is a function $\mapping{\samplesize_0}{\R_+\times\R_+}{\R_+}$ such that for all $\eps > 0$, $\conf\in(0, 1]$, $\samplesize\in\N$ with $\samplesize \geq \samplesize_0(\eps,\conf)$, and training data $\dataset\sim\dist^N$, the probability of generating an $\eps$-bad hypothesis is no greater than $\conf$, i.e., 
\begin{equation}
\prob\left(\qualityOf{\algo(\dataset)}>\eps\right)\leq \conf\enspace .
\label{eq:epsdeltaguarantee}
\end{equation}
They are efficient if the \defemph{sample complexity} $\samplesize_0(\eps,\conf)$ is polynomial in $\sfrac{1}{\eps}, \log\sfrac{1}{\conf}$ and the \defemph{runtime complexity} $\runtime_\algo$ is polynomial in the sample complexity.
This paper considers the parallelisation of such consistent and efficient learning algorithms, e.g., support vector machines, regularised least squares regression, and logistic regression.  We additionally assume that data is abundant and that
$\solutionset$ can be parametrised in a fixed, finite dimensional Euclidean space $\R^d$ such that the convexity of the regularised risk minimisation problem (Equation~\ref{eq:ml}) is preserved.
In other cases, (non-linear) low-dimensional embeddings~\citep{balcan2016kernelpca,oglic2017nystrom} can preprocess the data to facilitate parallel learning with our scheme.
With slight abuse of notation, we identify the hypothesis space with its parametrisation.

%

The main theoretical contribution of this paper is to show that algorithms satisfying the above conditions can be parallelised 
\defemph{effectively}. We consider a parallelisation to be effective if the $(\eps,\conf)$-guarantees (Equation~\ref{eq:epsdeltaguarantee}) are achieved in time polylogarithmic in $\samplesize_0(\eps,\conf)$. 
The cost for achieving this reduction in runtime comes in the form of an increased data size and in the number of processing units used. For the parallelisation scheme presented in this paper, we are able to bound this cost by a quasi-polynomial in $\sfrac{1}{\eps}$ and $\log\sfrac{1}{\conf}$.
The main practical contribution of this paper is an effective parallelisation scheme that treats the underlying learning algorithm as a \defemph{black-box}, i.e., it can be parallelised without further mathematical derivations and without writing dedicated code.

Similar to averaging-based parallelisations~\citep{rosenblatt2016optimality, zinkevich/nips/2010, zhang_communication-efficient_2013}, we apply the underlying learning algorithm in parallel to random subsets of the data.
Each resulting hypothesis is assigned to a leaf of an aggregation tree which is then traversed bottom-up. Each inner node computes a new hypothesis that is a \defemph{Radon point}~\citep{radon1921mengen} of its children's hypotheses. 
In contrast to aggregation by averaging, the Radon point increases the confidence in the aggregate doubly-exponentially with the height of the aggregation tree.
We describe our parallelisation scheme, the~\emph{\parallelSchemeName}, in detail in Section~\ref{se:rp}.
Comparing the Radon machine to the underlying learning algorithm which is applied to a dataset of the size necessary to achieve the same confidence, we are able to show a reduction in runtime from polynomial to polylogarithmic in Section~\ref{se:rt}.

The empirical evaluation of the~\parallelSchemeName~in Section~\ref{sec:experiments} confirms its potential in practical settings. Given the same amount of data as the underlying learning algorithm, the~\parallelSchemeName~achieves a substantial reduction of computation time in realistic applications. Using $150$ processors, the~\parallelSchemeName~is between $80$ and around $700$-times faster
than the underlying learning algorithm on a single processing unit. Compared with parallel learning algorithms from Spark's MLlib, it achieves hypotheses of similar quality, while requiring only $15-85\%$ of their runtime.

\begin{algorithm}[b]
	\caption{Radon Machine}	
	\begin{algorithmic}[1]
		\REQUIRE{learning algorithm $\algo$, dataset $\dataset\subseteq\inputspace\times\outputspace$, Radon number $r\in\N$, and parameter $h\in\N$}
		\ENSURE{{hypothesis $\solution\in\solutionset$}}
		\STATE \textbf{divide} $\dataset$ into $r^h$ iid subsets $\dataset_i$ of roughly equal size\label{algo:lineDivideDataset}
		\STATE \textbf{run} $\algo$ in parallel to obtain $\solution_i=\algo(\dataset_i)$\label{algo:lineRun}
		\STATE $S \leftarrow \{\solution_1,\dots,\solution_{r^h}\}$\label{algo:lineDefineSet}	
		\FOR{$i=h-1,\dots,1$}\label{algo:forBegin}
		\STATE \textbf{partition} $S$ into iid subsets $S_1,\dots,S_{r^i}$ of size $r$\label{algo:linePartition} each
		\STATE \textbf{calculate} Radon points $\radonPoint(S_1),\dots,\radonPoint(S_{r^i})$ in parallel\label{algo:lineCalcRadon} \hspace{\stretch{1}}\Comment{see Definition~\ref{def:radonPoint} and Appendix~\ref{sec:app:radonPointConstr}}
		\STATE $S\leftarrow\{\radonPoint(S_1),\dots,\radonPoint(S_{r^i})\}$\label{algo:lineReplaceWithRadon}
		\ENDFOR\label{algo:forEnd}
		\STATE \textbf{return}  $\radonPoint(S)$
	\end{algorithmic}	
	\label{alg:radonParallel}
\end{algorithm}
%
Parallel computing~\citep{kumar1994introduction} and its limitations~\citep{greenlaw_limits_1995} have been studied for a long time in theoretical computer science~\citep{chandra1976alternation}. Parallelising polynomial time algorithms ranges from being `embarrassingly'~\citep{moler1986matrix} easy to being believed to be impossible. For the class of decision problems that are the hardest in P, i.e., for P-complete problems, it is believed that there is no efficient parallel algorithm in the sense of Nick's Class (\NC~\cite{cook1979deterministic}):
efficient parallel algorithms in this sense are those that can be executed in \emph{polylogarithmic time} on a \emph{polynomial number of processing units}.
Our paper thus contributes to understanding the extent to which efficient parallelisation of polynomial time learning algorithms is possible.
This connection and other approaches to parallel learning are discussed in Section~\ref{sec:discussion}.

\section{From Radon Points to Radon Machines}\label{se:rp}
%
%
The~\parallelSchemeName, described in Algorithm~\ref{alg:radonParallel}, first executes the underlying (base) learning algorithm on random subsets of the data to quickly achieve weak hypotheses and then iteratively aggregates them to stronger ones. Both the generation of weak hypotheses and the aggregation can be executed in parallel.
To aggregate hypotheses, we follow along the lines of the iterated Radon point algorithm which was originally devised to approximate the centre point, i.e., a point of largest Tukey depth~\citep{tukey1975mathematics}, of a finite set of points~\cite{clarkson1996approximating}. The Radon point~\cite{radon1921mengen} of a set of points is defined as follows:
\begin{defin}
	A \defemph{Radon partition} of a set $S\subset\solutionset$ is a pair $A,B\subset S$ such that $A\cap B = \emptyset$ but $\convhull{A}\cap\convhull{B}\neq\emptyset$, where $\convhull{\cdot}$ denotes the convex hull. The \defemph{Radon number} of a space $\solutionset$ is the smallest $r\in\N$ such that for all $S\subset\solutionset$ with $|S| \geq r$ there is a Radon partition; or $\infty$ if no $S\subset\solutionset$ with Radon partition exists. A \defemph{Radon point} of a set $S$ with Radon partition $A,B$ is any $\radonPoint\in\convhull{A}\cap\convhull{B}$.
	\label{def:radonPoint}
\end{defin}
We now present the~\parallelSchemeName~(Algorithm~\ref{alg:radonParallel}), which is able to effectively parallelise consistent and efficient learning algorithms.
Input to this parallelisation scheme is a learning algorithm $\algo$ on a hypothesis space $\solutionset$, a dataset $\dataset\subseteq\inputspace\times\outputspace$, the Radon number $r\in\N$ of the hypothesis space $\solutionset$, and a parameter $h\in\N$. It divides the dataset into $r^h$ subsets $\dataset_1,\dots,\dataset_{r^h}$ (line~\ref{algo:lineDivideDataset}) and runs the algorithm $\algo$ on each subset in parallel (line~\ref{algo:lineRun}). Then, the set of hypotheses (line~\ref{algo:lineDefineSet}) is iteratively aggregated to form better sets of hypotheses (line~\ref{algo:forBegin}-\ref{algo:forEnd}). For that the set is partitioned into subsets of size $r$ (line~\ref{algo:linePartition}) and the Radon point of each subset is calculated in parallel (line~\ref{algo:lineCalcRadon}). The final step of each iteration is to replace the set of hypotheses by the set of Radon points (line~\ref{algo:lineReplaceWithRadon}). 

The scheme requires a hypothesis space with a valid notion of convexity and finite Radon number.
While other notions of convexity are possible~\cite{rubinov2013abstract,kay1971axiomatic}, in this paper we restrict our consideration to Euclidean spaces with the usual notion of convexity. Radon's theorem~\cite{radon1921mengen} states that the Euclidean space $\R^d$ has Radon number $r=d+2$. Radon points can then be obtained by solving a system of linear equations of size $r\times r$ (to be fully self-contained we state the system of linear equations explicitly in Appendix~\ref{sec:app:radonPointConstr}). 
%
The next proposition gives a guarantee on the quality of Radon points:
\begin{prop}
	Given a probability measure $\prob$ over a hypothesis space $\solutionset$ with finite Radon number $r$, let $F$ denote a random variable with distribution $\prob$. Furthermore, let $\radonPoint$ be the random variable obtained by computing the Radon point of $r$ random points drawn according to $\prob^r$. Then it holds for the expected regret $\quality$ and all $\eps\in\R$ that 
	\[
	\prob\left(\qualityOf{\radonPoint}>\eps\right)\leq\left(r\prob\left(\qualityOf{F}>\eps\right)\right)^2\enspace .
	\]
	\label{prop:probOfRadonBeingBad}
\end{prop}
\vspace{-0.5cm}
%
The proof of Proposition~\ref{prop:probOfRadonBeingBad} is provided in Section~\ref{sec:appendix:proofOfPropProbOfRadonBeingBad}.
Note that this proof also shows the robustness of the Radon point compared to the average: if only one of $r$ points is $\eps$-bad, the Radon point is still $\eps$-good, while the average may or may not be; indeed, in a linear space with any set of $\eps$-good hypotheses and any $\eps'\geq\eps$, we can always find a single $\eps'$-bad hypothesis such that the average of all these hypotheses is $\eps'$-bad.

A direct consequence of Proposition~\ref{prop:probOfRadonBeingBad} is a bound on the probability that the output of the~\parallelSchemeName~with parameter $h$ is bad:
\begin{thm}
	Given a probability measure $\prob$ over a hypothesis space $\solutionset$ with finite Radon number $r$, let $F$ denote a random variable with distribution $\prob$. Denote by $\radonPoint_1$ the random variable obtained by computing the Radon point of $r$ random points drawn iid according to $\prob$ and by $\prob_1$ its distribution. For any $h\in\N$, let $\IRP$ denote the Radon point of $r$ random points drawn iid from $\prob_{h-1}$ and by $\prob_{h}$ its distribution.
	Then for any convex function $\quality:\solutionset\rightarrow\R$ and all $\eps\in\R$ it holds that 
	\[
	\prob\left(\quality(\IRP)>\eps\right)\leq\left(r \prob\left(\quality(F)>\eps\right)\right)^{2^h}\enspace .
	\]
	\label{thm:confBoostofIteratedRadonPoint}
\end{thm}
\vspace{-0.5cm}
The proof of Theorem~\ref{thm:confBoostofIteratedRadonPoint} is also provided in Section~\ref{sec:appendix:proofOfPropProbOfRadonBeingBad}.
%
For the~\parallelSchemeName~with parameter $h$, Theorem~\ref{thm:confBoostofIteratedRadonPoint} shows that the probability of obtaining an $\eps$-bad hypothesis is doubly exponentially reduced: with a bound $\confBase$ on this probability for the base learning algorithm, the bound $\conf$ on this probability for the~\parallelSchemeName~is
\begin{equation}
\begin{split}
\conf = \left(r\confBase\right)^{2^h}\enspace .
\end{split}
\label{eq:confidenceFromConfBase}
\end{equation}
In the next section we 
compare the~\parallelSchemeName~to its base learning algorithm which is applied to a dataset of the size necessary to achieve the same $\eps$ and $\conf$.


\section{Sample and Runtime Complexity}
\label{se:rt}
In this section we first derive the sample and runtime complexity of the~\parallelSchemeName~$\radonAlgo$ from the sample and runtime complexity of the base learning algorithm $\algo$. We then relate the runtime complexity of the~\parallelSchemeName~to an application of the base learning algorithm which achieves the same $(\eps,\conf)$-guarantee.
%
%
For that, we consider consistent and efficient base learning algorithms with a sample complexity of the form $\samplesize^\algo_0(\eps,\confBase) = \left(\alpha_{\eps}+\beta_{\eps}\ld\sfrac{1}{\confBase}\right)^{k}$, 
for some\footnote{We derive $\alpha_{\eps},\beta_{\eps}$ for hypothesis spaces with finite VC~\cite{vapnik1971uniform} and Rademacher~\cite{bartlett2003rademacher} complexity in App.~\ref{sec:appendix:consistency}.}
$\alpha_{\eps},\beta_{\eps}\in\R$, and $k\in\N$. 
From now on, we also assume that $\confBase \leq \sfrac{1}{2r}$ for the base learning algorithm.

%

The~\parallelSchemeName~creates $r^h$ base hypotheses and, with $\conf$ as in Equation~\ref{eq:confidenceFromConfBase}, has sample complexity 
\begin{equation}
\samplesize_0^\radonAlgo(\eps,\conf)\ =\ r^h\samplesize^\algo_0(\eps,\confBase)\ =\ r^h\cdot\left(\alpha_{\eps}\ +\ \beta_{\eps}\ld\frac{1}{\confBase}\right)^{k}\enspace .
\end{equation}
Theorem~\ref{thm:confBoostofIteratedRadonPoint} then implies that the~\parallelSchemeName~with base learning algorithm $\algo$ is consistent: with $\samplesize \geq \samplesize_0^\radonAlgo(\eps,\conf)$ samples it achieves an $(\eps,\conf)$-guarantee. 

To achieve the same guarantee as the~\parallelSchemeName, the application of the base learning algorithm $\algo$ itself (sequentially) 
would require $M\geq\samplesize^\algo_0(\eps,\conf)$ samples, where
\begin{equation}
\samplesize^\algo_0(\eps,\conf)\ =\ \samplesize^\algo_0\left(\eps,(r\confBase)^{2^h}\right)\ =\ \left(\alpha_{\eps}\ \ +\ \ 2^h\cdot\beta_{\eps}\ld\frac{1}{r\confBase}\right)^{k}\enspace .
\label{eq:sampleSizeofSequential}
\end{equation}

For base learning algorithms $\algo$ with runtime $\runtime_\algo(\basesamplesize)$ polynomial in the data size $\basesamplesize\in\N$, i.e., ${\runtime_\algo(\basesamplesize)\in\bigo\left(\basesamplesize^\kappa\right)}$ with $\kappa\in\N$, we now determine the runtime $\runtime_{\radonAlgo,h}(\samplesize)$ of the~\parallelSchemeName~with $h$ iterations and $c=r^h$ processing units on $\samplesize\in\N$ samples. In this case all base learning algorithms can be executed in parallel. In practical applications fewer physical processors can be used to simulate $r^h$ processing units---we discuss this case in Section~\ref{sec:discussion}.

The runtime of the~\parallelSchemeName~can be decomposed into the runtime of the base learning algorithm and the runtime for the aggregation. 
The base learning algorithm requires $\basesamplesize\geq\samplesize_0^\algo(\eps,\confBase)$ samples and can be executed on $r^h$ processors in parallel in time $\runtime_\algo(\basesamplesize)$.
The Radon point in each of the $h$ iterations can then be calculated in parallel in time $r^3$~(see Appendix~\ref{sec:app:radonPointConstr}).
Thus, the runtime of the~\parallelSchemeName~with $\samplesize=r^h\basesamplesize$ samples is 
\begin{equation}
\runtime_{\radonAlgo,h}(\samplesize)=\runtime_\algo\left(\basesamplesize\right)+hr^3\enspace .
\label{eg:runtimeRadonAlgo}
\end{equation}
In contrast, the runtime of the base learning algorithm for achieving the same guarantee is $\runtime_{\algo}(M)$ with $M\geq\samplesize^\algo_0(\eps,\conf)$. Ignoring logarithmic and constant terms, $\samplesize_0^{\algo}(\eps,\conf)$ behaves as $2^{h}\samplesize_0^{\algo}(\eps,\confBase)$.
%
To obtain polylogarithmic runtime of $\radonAlgo$ compared to $\runtime_{\algo}(M)$, we choose the parameter ${h\approx\ld M - \ld\ld M}$ such that $n\approx \sfrac{M}{2^h}=\ld M$. Thus, the runtime of the~\parallelSchemeName~is in $\bigo\left(\ld^\kappa M + r^3\ld M \right)$. 
%
This result is formally summarised in Theorem~\ref{prop:polylogRuntime}.
\begin{thm}
	The~\parallelSchemeName~with a consistent and efficient regularised risk minimisation algorithm on a hypothesis space with finite Radon number has polylogarithmic runtime on quasi-polynomially many processing units if the Radon number can be upper bounded by a function polylogarithmic in the sample complexity of the efficient regularised risk minimisation algorithm.
	\label{prop:polylogRuntime}
\end{thm}
The theorem is proven in Appendix~\ref{sec:appendix:proofThmPolylogRuntime} and relates to Nick's Class~\citep{arora2009computational}:
A decision problem can be solved efficiently in parallel in the sense of Nick's Class, if it can be decided by an algorithm in polylogarithmic time on polynomially many processors (assuming, e.g., PRAM model). 
For the class of decision problems that are the hardest in $P$, i.e., for $P$-complete problems, it is believed that there is no efficient parallel algorithm for solving them in this sense.
Theorem~\ref{prop:polylogRuntime} provides a step towards finding efficient parallelisations of regularised risk minimisers and towards answering the open question: is consistent regularised risk minimisation possible in polylogarithmic time on polynomially many processors. 
A similar question, for the case of learning half spaces, has been called a fundamental open problem by~\citet{long_algorithms_2013} who gave an algorithms which runs on polynomially many processors in time that depends polylogarithmically on the sample size but is inversely proportional to a parameter of the learning problem. 
While Nick's Class as a notion of efficiency has been criticised~\citep{kruskal1990complexity}, it is the only notion of efficiency that forms a proper complexity class in the sense of~\citet{blum1967machine}. To overcome the weakness of using only this notion, \citet{kruskal1990complexity} suggested to consider also the inefficiency of simulating the parallel algorithm on a single processing unit. We discuss the inefficiency and the speed-up in Appendix~\ref{sec:appendix:proofPropRuntime}. 

\section{Empirical Evaluation}
\label{sec:experiments}
This empirical study compares the~\parallelSchemeName~to state-of-the-art parallel machine learning algorithms from the Spark machine learning library~\cite{meng2016mllib}, 
as well as the natural baseline of averaging hypotheses instead of calculating their Radon point (averaging-at-the-end, \defemph{Avg}). 
We use base learning algorithms from WEKA~\citep{witten2016data} and scikit-learn~\cite{scikit-learn}. We compare the~\parallelSchemeName~to the base learning algorithms on moderately sized datasets, due to scalability limitations of the base learners, and reserve larger datasets for the comparison with parallel learners. The experiments are executed on a Spark cluster ($5$ worker nodes, $25$ processors per node)\footnote{The source code implementation in Spark can be found in the bitbucket repository\hspace{\stretch{1}}\\\url{https://bitbucket.org/Michael_Kamp/radonmachine}.}. All results are obtained using $10$-fold cross validation. We apply the~\parallelSchemeName~with parameter $h=1$ and the maximal parameter $h$ such that each instance of the base learning algorithm is executed on a subset of size at least $100$ (denoted $h=max$). Averaging-at-the-end executes the base learning algorithm on the same number of subsets $r^h$ as the~\parallelSchemeName~with that parameter and is denoted in the Figures by stating the parameter $h$ as for the~\parallelSchemeName. All other parameters of the learning algorithms are optimised on an independent split of the datasets. See Appendix~\ref{appendix:addExps} for additional details.
\begin{figure*}[b!]
	\centering
	\subfigure[]{\label{fig:expBoxPlotsRuntimeAUCCentral}\includegraphics[height=10.5cm]{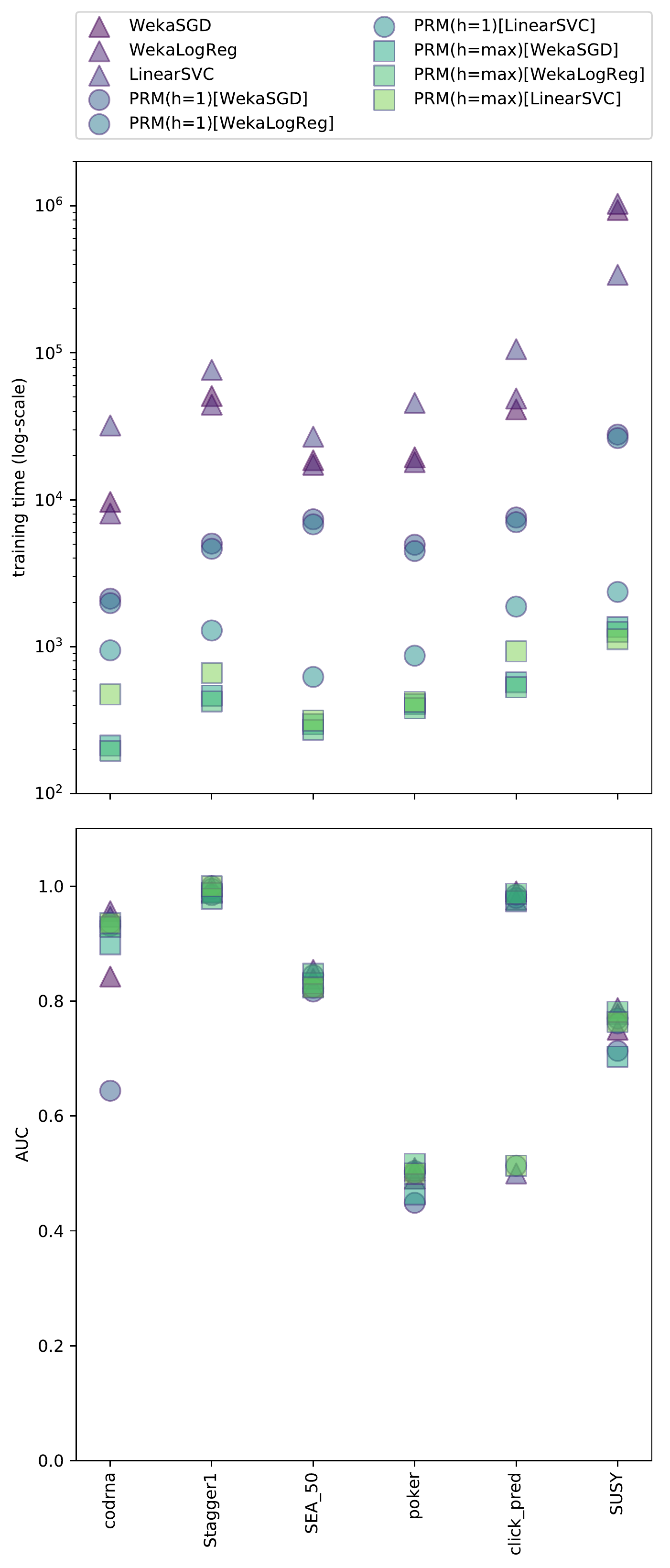}}\hfill%
	\subfigure[]{\label{fig:radonVsAverage}\includegraphics[height=10.5cm]{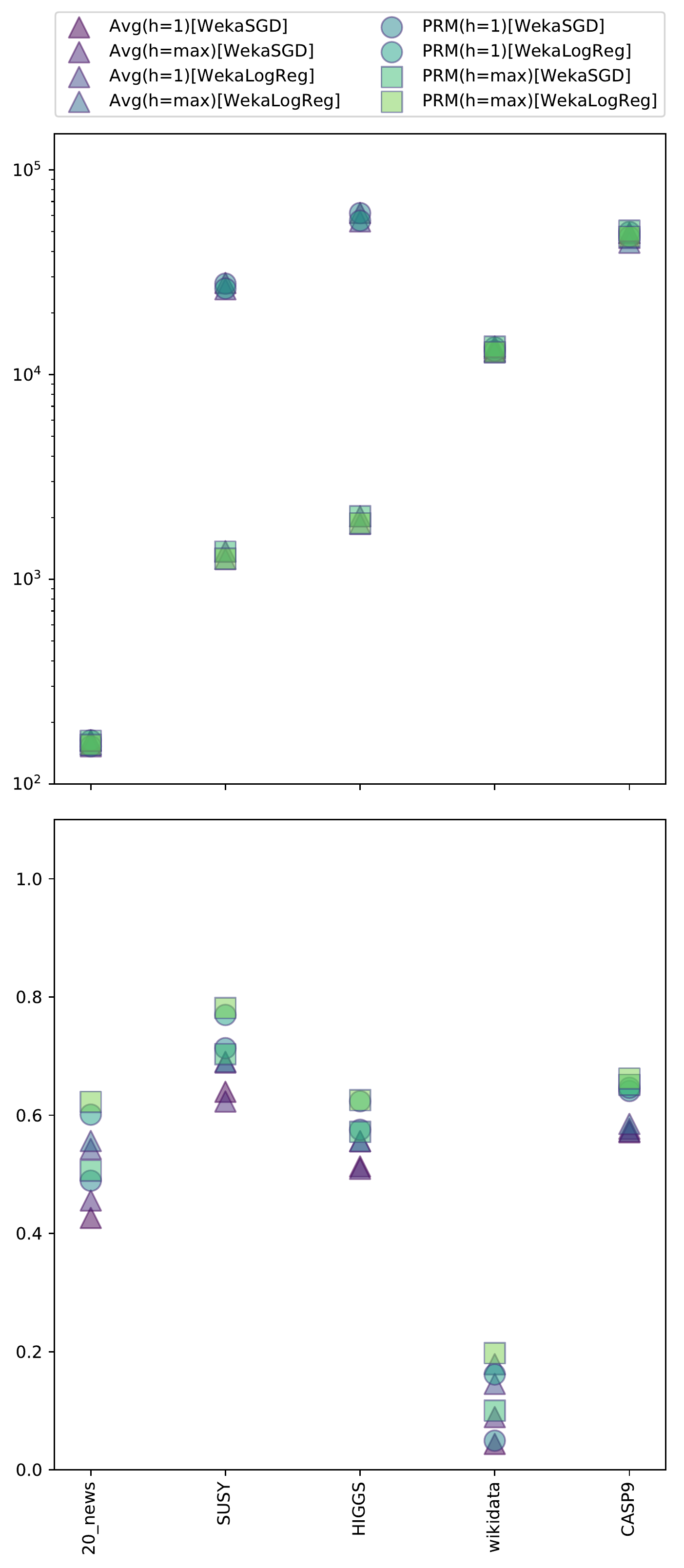}}\hfill%
	\subfigure[]{\label{fig:expBoxPlotsRuntimeAUCSpark}\includegraphics[height=10.5cm]{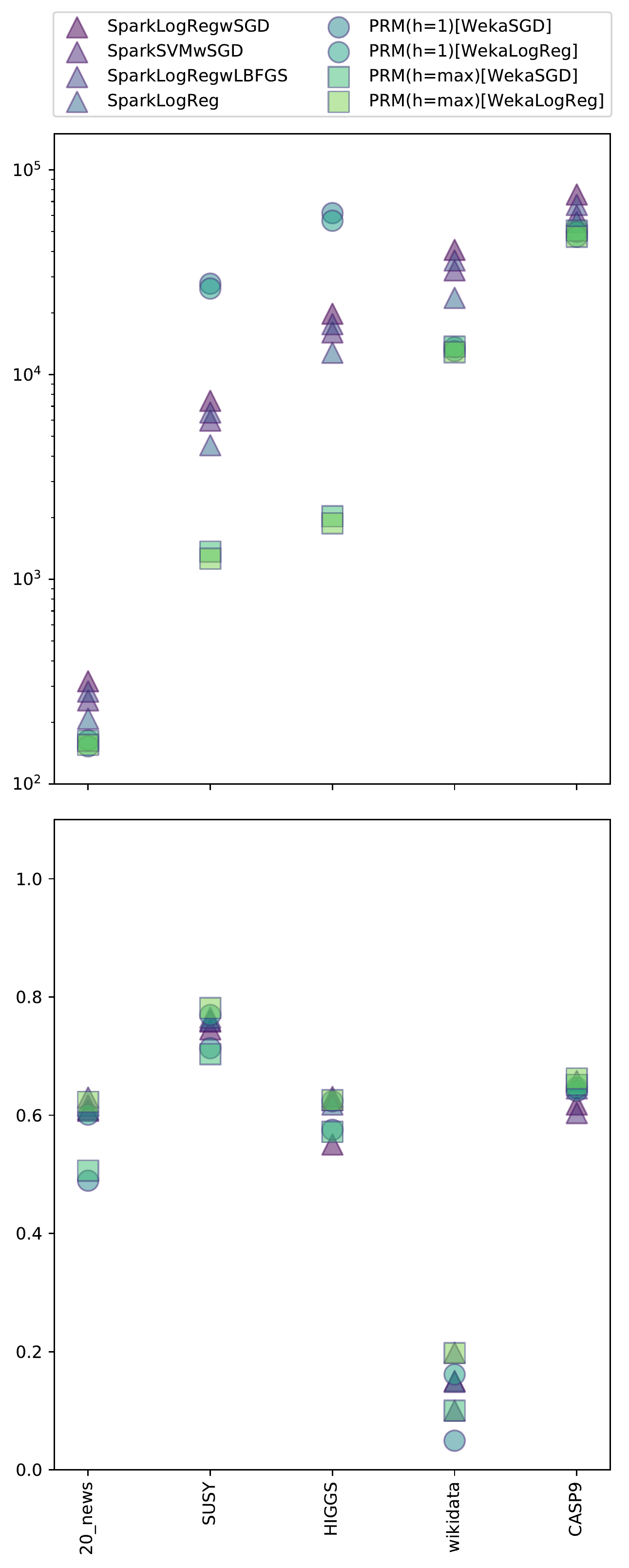}}\hfill%
	\caption{
		\subref{fig:expBoxPlotsRuntimeAUCCentral} Runtime (log-scale) and AUC of base learners and their parallelisation using the~\parallelSchemeName~(PRM) for $6$ datasets with $\samplesize\in [488\,565, 5\,000\,000]$, $d\in [3, 18]$. Each point represents the average runtime (upper part) and AUC (lower part) over $10$ folds of a learner---or its parallelisation---on one datasets.~		
		\subref{fig:radonVsAverage} Runtime and AUC of the~\parallelSchemeName~compared to the averaging-at-the-end baseline (Avg) on $5$ datasets with $\samplesize\in [5\,000\,000, 32\,000\,000]$, $d\in [18, 2\,331]$.~
		\subref{fig:expBoxPlotsRuntimeAUCSpark} Runtime and AUC of several Spark machine learning library algorithms and the~\parallelSchemeName~using base learners that are comparable to the Spark algorithms on the same datasets as in Figure~\ref{fig:radonVsAverage}.
	}
\end{figure*}
\begin{figure}
	\centering
	\begin{minipage}{.47\textwidth}
		\centering
		\includegraphics[height=4.5cm]{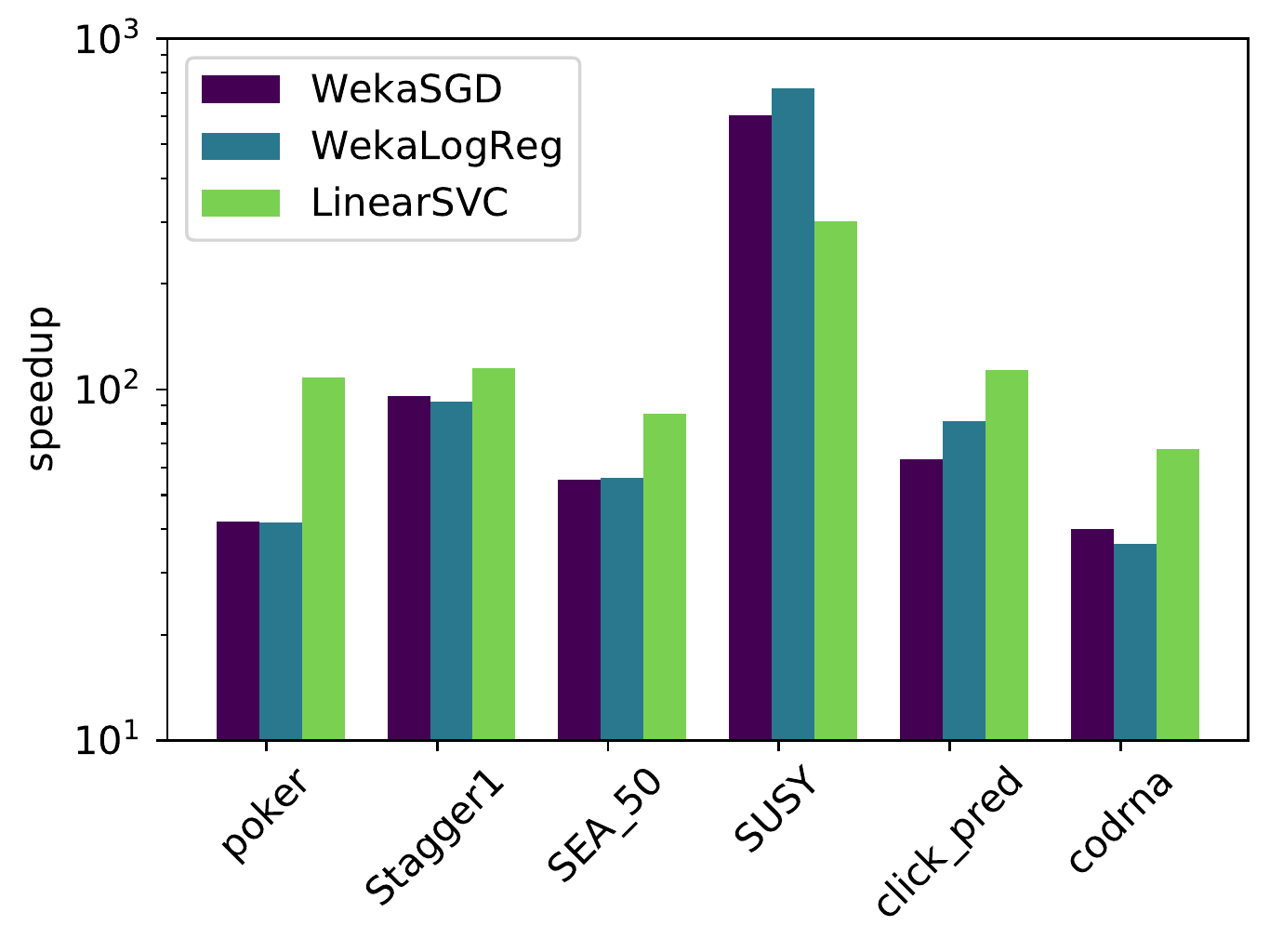}
		\captionof{figure}{Speed-up (log-scale) of the~\parallelSchemeName~over its base learners per dataset from the same experiment as in Figure~\ref{fig:expBoxPlotsRuntimeAUCCentral}.}
		\label{fig:speedupPerDataset}
	\end{minipage}%
	\hspace{0.5cm}
	\begin{minipage}{.47\textwidth}
		\centering
		\includegraphics[height=4.5cm]{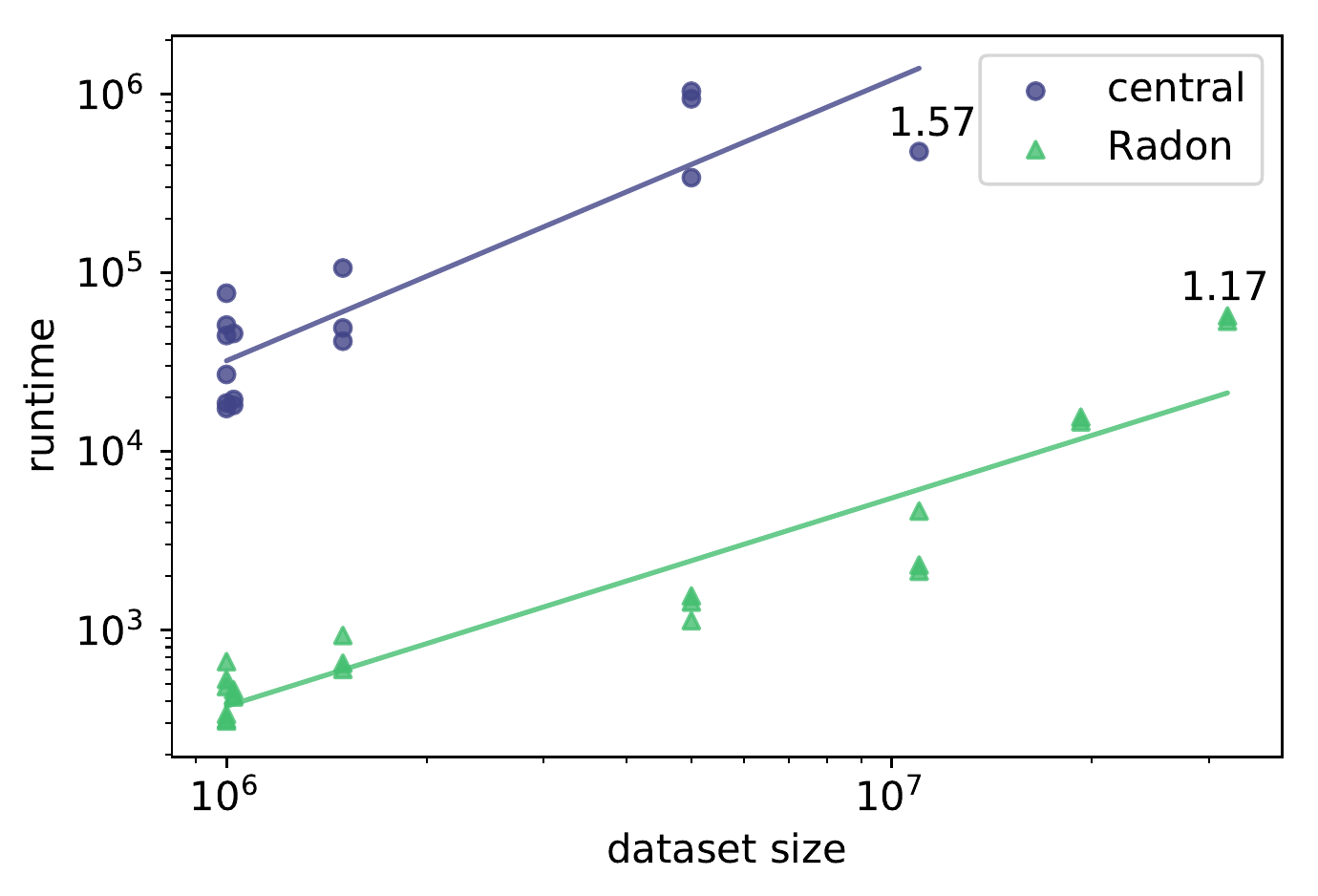}
		\captionof{figure}{Dependence of the runtime on the dataset size for of the~\parallelSchemeName~compared to its base learners.}
		\label{fig:expRuntimeVsDatasize}
	\end{minipage}
\end{figure}
%

\textbf{What is the speed-up of our scheme in practice?}
In Figure~\ref{fig:expBoxPlotsRuntimeAUCCentral}, we compare the~\parallelSchemeName~to its base learners on moderately sized datasets (details on the datasets are provided in Appendix~\ref{appendix:addExps}). 
\begin{wrapfigure}{r}{0.5\textwidth}
	\includegraphics[width=7.1cm]{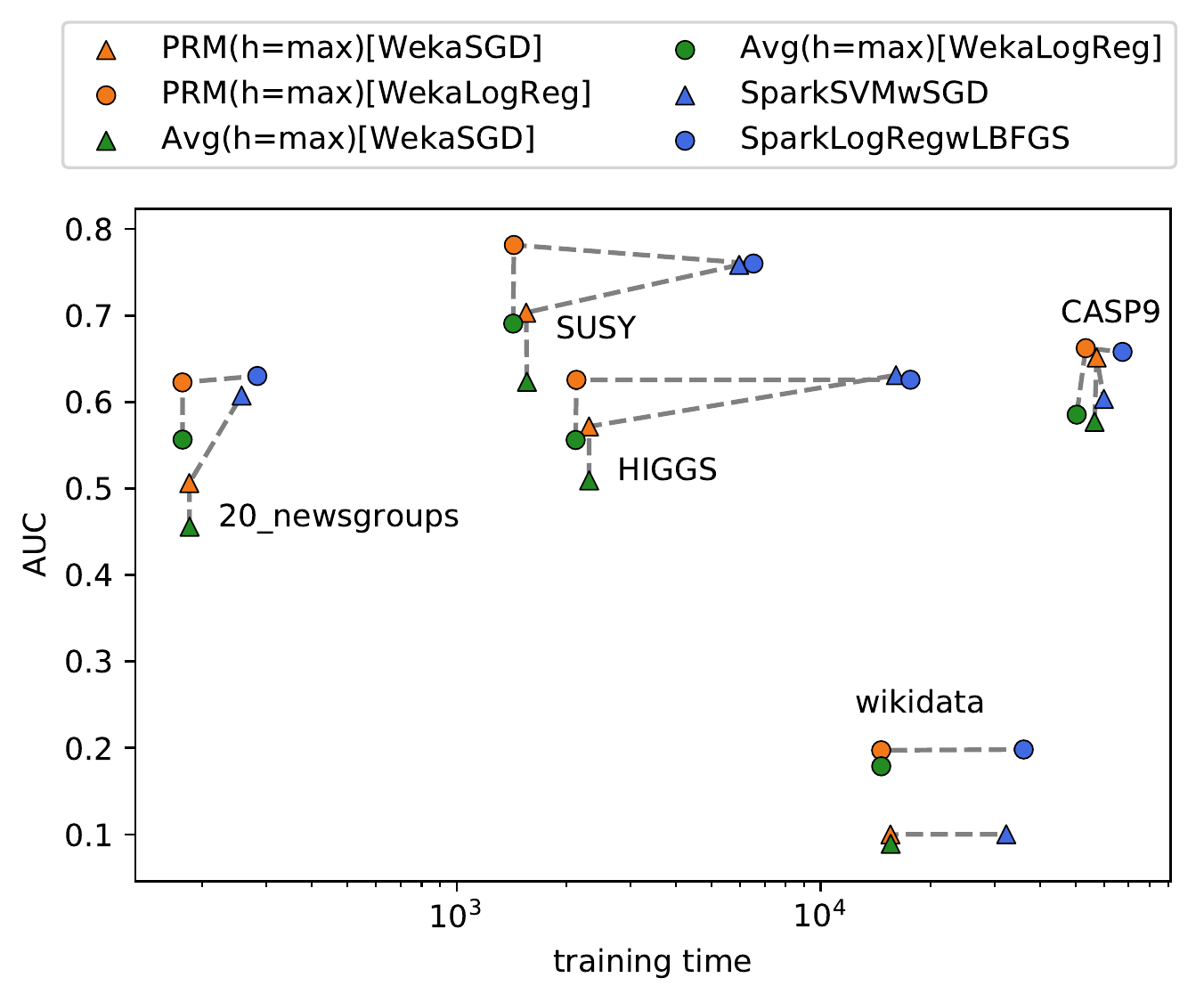}
	\captionof{figure}{Representation of the results in Figure~\ref{fig:radonVsAverage} and~\ref{fig:expBoxPlotsRuntimeAUCSpark} in terms of the trade-off between runtime and AUC for the~\parallelSchemeName~(PRM) and averaging-at-the-end (Avg), both with parameter $h=max$, and parallel machine learning algorithms in Spark. The dashed lines connect the~\parallelSchemeName~to averaging-at-the-end with the same base learning algorithm and a comparable Spark machine learning algorithm.}
	\label{fig:expAucVsRuntime}
\end{wrapfigure}
There, the~\parallelSchemeName~is between $80$ and around $700$-times faster than the base learner using $150$ processors. The speed-up is detailed in Figure~\ref{fig:speedupPerDataset}. On the SUSY dataset
(with $5\,000\,000$ instances and $18$ features), the~\parallelSchemeName~on $150$ processors with ${h=3}$ is $721$ times faster than its base learning algorithms. 
At the same time, their predictive performances, measured by the area under the ROC curve (AUC) on an independent test dataset, are comparable.

\textbf{How does the scheme compare to averaging-at-the-end?} In Figure~\ref{fig:radonVsAverage} we compare the runtime and AUC of the parallelisation scheme against the averaging-at-the-end baseline (\emph{Avg}). In terms of the AUC, the~\parallelSchemeName~outperforms the averaging-at-the-end baseline on all datasets by at least $10\%$. The runtimes can hardly be distinguished in that figure. A small difference can however be noted in Figure~\ref{fig:expAucVsRuntime} which is discussed in more details in the next paragraph. Since averaging is less computationally expensive than calculating the Radon point, the runtimes of the averaging-at-the-end baselines are slightly lower than the ones of the~\parallelSchemeName. However, compared to the computational complexity of executing the base learner, this advantage becomes negligible. 

%
%
\textbf{How does our scheme compare to state-of-the-art Spark machine learning algorithms?} 
We compare the~\parallelSchemeName~to various Spark machine learning algorithms on $5$ large datasets. The results in Figure~\ref{fig:expBoxPlotsRuntimeAUCSpark} indicate that the proposed parallelisation scheme with $h=max$ has a substantially smaller runtime than the Spark algorithms on all datasets. 
On the SUSY and HIGGS dataset, the~\parallelSchemeName~is one order of magnitude faster than the Spark implementations---here the comparatively small number of features allows for a high level of parallelism. On the CASP9 dataset, the~\parallelSchemeName~is $15$\% faster than the fastest Spark algorithm. The performance in terms of AUC of the~\parallelSchemeName~is similar to the Spark algorithms. In particular, when using WekaLogReg with $h=max$, the~\parallelSchemeName~outperforms the Spark algorithms in terms of AUC and runtime on the datasets SUSY, wikidata, and CASP9. Details are given in the Appendix~\ref{appendix:addExps}. A summarizing comparison of the parallel approaches in terms of their trade-off between runtime and predictive performance is depicted in Figure~\ref{fig:expAucVsRuntime}. Here, results are shown for the~\parallelSchemeName~and averaging-at-the-end with parameter $h=max$ and for the two Spark algorithms most similar to the base learning algorithms. Note that it is unclear what caused the consistently weak performance of all algorithms on wikidata. Nonetheless, the results show that on all datasets the~\parallelSchemeName~has comparable predictive performance to the Spark algorithms and substantially higher predictive performance than averaging-at-the-end. At the same time, the~\parallelSchemeName~has a runtime comparable to averaging-at-the-end on all datasets and both are substantially faster than the Spark algorithms.

\textbf{How does the runtime depend on the dataset size in a real-world system?} 
The runtime of the~\parallelSchemeName~can be distinguished into its learning phase and its aggregation phase. While the learning phase fully benefits from parallelisation, this comes at the cost of additional runtime for the aggregation phase. The time for aggregating the hypotheses does not depend on the number of instances in the dataset but for a fixed parameter $h$ it only depends on the dimension of the hypothesis space and that parameter.
In Figure~\ref{fig:expRuntimeVsDatasize} we compare the runtimes of all base learning algorithms per dataset size to the~\parallelSchemeName s. 
Results indicate that, while the runtimes of the base learning algorithms depends on the dataset size with an average exponent of $1.57$, the runtime of the~\parallelSchemeName~depends on the dataset size with an exponent of only $1.17$. 

\textbf{How generally applicable is the scheme?} As an indication of the general applicability in practice, we also consider regression and multi-class classification. For regression, we apply the scheme to the Scikit-learn implementation of regularised least squares regression~\cite{scikit-learn}. On the dataset \emph{YearPredictionMSD}, regularised least squares regression achieves an RMSE of $12.57$, whereas the~\parallelSchemeName~achieved an RMSE of $13.64$. At the same time, the~\parallelSchemeName~is $197$-times faster. We also compare the~\parallelSchemeName~on a multi-class prediction problem using conditional maximum entropy models. For multi-class classification, we use the implementation described in~\citet{mcdonald2009efficient}, who propose to use averaging-at-the-end for distributed training. We compare the~\parallelSchemeName~to averaging-at-the-end with conditional maximum entropy models on two large multi-class datasets (\emph{drift} and \emph{spoken-arabic-digit}). On average, our scheme performs 
better with only 
slightly longer runtime. The minimal difference in runtime can be explained---similar to the results in Figure~\ref{fig:radonVsAverage}---by the smaller complexity of calculating the average instead of the Radon point.

\section{Discussion and Future Work}
\label{sec:discussion}
%
%
%
%
%
In the experiments we considered datasets where the number of dimensions is much smaller than the number of instances. \textbf{What about high-dimensional models?} 
The basic version of the parallelisation scheme presented in this paper cannot directly be applied to cases in which the size of the dataset is not at least a multiple of the Radon number of the hypothesis space. For various types of data such as text, this might cause concerns. However, random projections~\cite{johnson1984extensions} or low-rank approximations~\cite{balcan2016kernelpca,oglic2017nystrom} can alleviate this problem and are already frequently employed in machine learning. An alternative might be to combine our parallelisation scheme with block coordinate descent~\cite{sra2012optimization}. In this case, the scheme can be applied iteratively to subsets of the features.

In the experiments we considered only linear models.
\textbf{What about non-linear models?} 
Learning non-linear models causes similar problems to learning high-dimensional ones. In non-parametric methods like kernel methods, for instance, the dimensionality of the optimisation problem is equal to the number of instances, thus prohibiting the application of our parallelisation scheme. However, similar low-rank approximation techniques as described above can be applied with non-linear kernels~\cite{fine2002efficient}. Furthermore, methods for speeding up the learning process for non-linear models explicitly approximate an embedding in which then a linear model can be learned~\cite{rahimi2007random}. Using explicitly constructed feature spaces, \parallelSchemeName s can directly be applied to non-linear models.

We have theoretically analysed our parallelisation scheme for the case that there are enough processing units available to find each weak hypothesis on a separate processing units. 
\textbf{What if there are not $r^h$, but only $c<r^h$ processing units?}
The parallelisation scheme can quite naturally be ``de-parallelised'' and partially executed in sequence. For the runtime this implies an additional factor of $\max\{1,\sfrac{r^h}{\cores}\}$. 
Thus, the~\parallelSchemeName~can be applied with any number of processing units.

The scheme improves $\conf$ doubly exponentially in its parameter $h$ but for that it requires the weak hypotheses to already achieve $\confBase\leq\sfrac{1}{2r}$. \textbf{Is the scheme only applicable in high-confidence domains?} Many application scenarios require high-confidence error bounds, e.g., in the medical domain~\citep{nouretdinov2011machine} or in intrusion detection~\citep{sommer2010outside}. 
In practice our scheme achieves similar predictive quality much faster than its base learner.

Besides runtime, communication plays an essential role in parallel learning. \textbf{What is the communication complexity of the scheme?} As for all aggregation at the end strategies, the overall amount of communication is low compared to periodically communicating schemes. For the parallel aggregation of hypotheses, the scheme requires $\bigo(r^{h+1})$ messages (which can be sent in parallel) each containing a single hypothesis of size $\bigo(r)$. 
Our scheme is ideally suited for inherently distributed data and might even mitigate privacy concerns.

In a lot of applications data is available in the form of potentially infinite data streams. \textbf{Can the scheme be applied to distributed data streams?} For each data stream, a hypotheses could be maintained using an online learning algorithm and periodically aggregated using the~\parallelSchemeName, similar to the federated learning approach proposed by~\citet{mcmahan2017communication}. 

In this paper, we investigated the parallelisation of machine learning algorithms.
\textbf{Is the~\parallelSchemeName~more generally applicable?} The parallelisation scheme could be applied to more general randomized convex optimization algorithms with unknown and random target functions. We will investigate its applicability for learning in non-Euclidean, abstract convexity spaces.

\section{Conclusion and Related Work}
\label{sec:conclusion}
In this paper we provided a step towards answering an open problem: \emph{Is parallel machine learning possible in polylogarithmic time using a polynomial number of processors only?} This question has been posed for half-spaces by~\citet{long_algorithms_2013} and called ``a fundamental open problem about the abilities and limitations of efficient parallel learning algorithms''. It relates machine learning to Nick's Class of parallelisable decision problems and its variants \cite{greenlaw_limits_1995}.
Early theoretical treatments of parallel learning with respect to NC considered \emph{probably approximately correct} (PAC) \cite{valiant_theory_1984,blumer1989learnability} concept learning. Vitter and Lin \cite{vitter1992learning} introduced the notion of \emph{NC-learnable} for concept classes for which there is an algorithm that outputs a probably approximately correct hypothesis in polylogarithmic time using a polynomial number of processors. In this setting, they proved positive and negative learnability results for a number of concept classes that were previously known to be PAC-learnable in polynomial time. 
%
More recently, the special case of learning half spaces in parallel was considered by Long and Servedio \cite{long_algorithms_2013} who gave an algorithm for this case that runs on polynomially many processors in time that depends polylogarithmically on the size of the instances but is inversely proportional to a parameter of the learning problem. 
Our paper complements these theoretical treatments of parallel machine learning and provides a provably effective parallelisation scheme for a broad class of regularised risk minimisation algorithms.

Some parallelisation schemes also train learning algorithms on small chunks of data and average the found hypotheses. While this approach has advantages \cite{freund_why_2001,rosenblatt2016optimality}, current error bounds do not allow a derivation of polylogarithmic runtime~\cite{zhang_communication-efficient_2013,lin2017distributed, shamir2014communication} and it has been doubted to have any benefit over learning on a single chunk \cite{shamir2014distributed}. 
Another popular class of parallel learning algorithms is based on stochastic gradient descent, targeting expected risk minimisation directly \citep[and references therein]{shamir2014distributed}. The best so far known algorithm in this class \citep{shamir2014distributed} is the distributed mini-batch algorithm \citep{dekel2012optimal}. This algorithm still runs for a number of rounds inversely proportional to the desired optimisation error, hence not in polylogarithmic time. 
A more traditional approach is to minimise the \emph{empirical risk}, i.e., an empirical sample-based approximation of the expected risk, using any, deterministic or randomised, optimisation algorithm. This approach relies on generalisation guarantees relating the expected and empirical risk minimisation as well as a guarantee on the optimisation error introduced by the optimisation algorithm. The approach is readily parallelisable by employing available parallel optimisation algorithms \cite[e.g.,][]{boyd2011distributed}. It is worth noting that these algorithms solve a harder than necessary optimisation problem and often come with prohibitively high communication cost in distributed settings \cite{shamir2014distributed}. Recent results improve over these \cite{ma_distributed_2017} but cannot achieve polylogarithmic time as the number of iterations depends linearly on the number of processors. 

Apart from its theoretical advantages, the~\parallelSchemeName~also has several practical benefits. In particular, it is a black-box parallelisation scheme in the sense that it is applicable to a wide range of machine learning algorithms and it does not depend on the implementation of these algorithms. It speeds up learning while achieving a similar hypothesis quality as the base learner. Our empirical evaluation indicates that in practice the~\parallelSchemeName~achieves either a substantial speed-up or a higher predictive performance than other parallel machine learning algorithms.


\vspace{\stretch{1}}
\pagebreak

\section{Proof of Proposition~\ref{prop:probOfRadonBeingBad} and Theorem~\ref{thm:confBoostofIteratedRadonPoint}}
\label{sec:appendix:proofOfPropProbOfRadonBeingBad}
In order to prove Proposition~\ref{prop:probOfRadonBeingBad} and consecutively Theorem~\ref{thm:confBoostofIteratedRadonPoint}, we first investigate some properties of Radon points and convex functions. 
We proof these properties for the more general case of quasi-convex functions. Since every convex function is also quasi-convex, the results hold for convex functions as well. A quasi-convex function is defined as follows.
\begin{defin}
	A function $\quality:\solutionset\rightarrow\R$ is called \emph{quasi-convex} if all its sublevel sets are convex, i.e.,
	\[
	\forall \theta\in\R :\{\solution\in\solutionset \mid \qualityOf{\solution} < \theta\}\text{ is convex.}
	\]
\end{defin}
First we give a different characterisation of quasi-convex functions. 
\begin{prop}
	A function $\quality:\solutionset\rightarrow\R$ is quasi-convex if and only if for all $S\subseteq\solutionset$ and all $s'\in\convhull{S}$ there exists an $s\in S$ with $\qualityOf{s}\geq \qualityOf{s'}$.
	\label{prop:convFuncOneWorseThanRadon}
\end{prop}
\begin{proof}
	~
	\begin{itemize}
		\item[$(\Rightarrow)$]Suppose this direction does not hold. Then there is a quasi-convex function $\quality$, a set $S\subseteq\solutionset$, and an $s'\in \convhull{S}$ such that for all $s\in S$ it holds that $\qualityOf{s} < \qualityOf{s'}$ (therefore $s'\notin S$). Let $C=\{c\in\solutionset\mid\qualityOf{c} < \qualityOf{s'}\}$. As $S\subseteq C=\convhull{C}$ we also have that $\convhull{S}\subseteq\convhull{C}$ which contradicts $\convhull{S}\ni s'\notin C$.
		\item[$(\Leftarrow)$]Suppose this direction does not hold. Then there exists an $\eps$ such that ${S=\{s\in\solutionset\mid\qualityOf{s}<\eps\}}$ is not convex and therefore there is an $s'\in\convhull{S}\setminus S$. By assumption $\exists s\in S:\qualityOf{s}\geq\qualityOf{s'}$. Hence $\qualityOf{s'} < \eps$ and we have a contradiction since this would imply $s'\in S$.
	\end{itemize}
\end{proof}
The next proposition concerns the value of any convex function at a Radon point.
\begin{prop}
	For every set $S$ with Radon point $\radonPoint$ and every quasi-convex function $\quality$ it holds that 
	$\left| \{ s\in S \mid \qualityOf{s} \geq \qualityOf{\radonPoint} \} \right| \geq 2$.
	\label{prop:numberOfWorsePointsThanRadon}
\end{prop}
\begin{proof}
	We show a slightly stronger result: Take any family of pairwise disjoint sets $A_i$ with $\bigcap_i \convhull{A_i}\neq\emptyset$ and $\radonPoint\in\bigcap_i \convhull{A_i}$. From proposition~\ref{prop:convFuncOneWorseThanRadon} follows directly the existence of an $a_i\in A_i$ such that $\qualityOf{a_i} \geq \qualityOf{\radonPoint}$. The desired result follows then from $a_i\neq a_j\Leftarrow i\neq j$.
\end{proof}
We are now ready to proof Proposition~\ref{prop:probOfRadonBeingBad} and Theorem~\ref{thm:confBoostofIteratedRadonPoint} (which we re-state here for convenience).
\addtocounter{thmrestate}{2}
\begin{thmrestate}
	Given a probability measure $\prob$ over a hypothesis space $\solutionset$ with finite Radon number $r$, let $F$ denote a random variable with distribution $\prob$. Denote by $\radonPoint_1$ the random variable obtained by computing the Radon point of $r$ random points drawn iid according to $\prob$ and by $\prob_1$ its distribution. For any $h\in\N$, let $\IRP$ denote the Radon point of $r$ random points drawn iid from $\prob_{h-1}$ and by $\prob_{h}$ its distribution.
	Then for any convex function $\quality:\solutionset\rightarrow\R$ and all $\eps\in\R$ it holds that 
	\[
	\prob\left(\quality(\IRP)>\eps\right)\leq\left(r \prob\left(\quality(F)>\eps\right)\right)^{2^h}\enspace .
	\]
\end{thmrestate}
\begin{proof}[Proof of Proposition~\ref{prop:probOfRadonBeingBad} and Theorem~\ref{thm:confBoostofIteratedRadonPoint}]
	By proposition~\ref{prop:numberOfWorsePointsThanRadon}, for any Radon point $\radonPoint$ of a set $S$ there must be two points $a,b\in S$ with $\qualityOf{a},\qualityOf{b} \geq \qualityOf{\radonPoint}$. Henceforth, the probability of $\qualityOf{\radonPoint} > \eps$ is less than or equal to the probability of the pair $a,b$ having $\qualityOf{a},\qualityOf{b} > \eps$. Proposition~\ref{prop:probOfRadonBeingBad} follows by an application of the union bound on all pairs from $S$.
	Repeated application of the proposition proves Theorem~\ref{thm:confBoostofIteratedRadonPoint}.
\end{proof}

\subsubsection*{Acknowledgements}
Part of this work was conducted while Mario Boley, Olana Missura, and Thomas G{\"a}rtner were at the University of Bonn and partially funded by the German Science Foundation (DFG, under ref. GA 1615/1-1 and GA 1615/2-1). 
The authors would like to thank Dino Oglic, Graham Hutton, Roderick MacKenzie, and Stefan Wrobel for valuable discussions and comments.

\newpage
\appendix
\label{sec:appendix}
\section{Theory}

\subsection{Proof of Theorem~\ref{prop:polylogRuntime}}
\label{sec:appendix:proofThmPolylogRuntime}
In the following, Theorem~\ref{prop:polylogRuntime} is proven we which re-state here for convenience.
\begin{thmrestate}
	The~\parallelSchemeName~with a consistent and efficient regularised risk minimisation algorithm on a hypothesis space with finite Radon number has polylogarithmic runtime on quasi-polynomially many processing units if the Radon number can be upper bounded by a function polylogarithmic in the sample complexity of the efficient regularised risk minimisation algorithm.y
\end{thmrestate}
\begin{proof}
	We assume the base learning algorithm $\algo$ to be a  consistent and efficient regularised risk minimisation algorithm on a hypothesis space with finite Radon number. 
	Let $r\in\N$ be the Radon number of the hypothesis space and
	\[
	\samplesize_0^\algo(\eps, \confBase)=\left(\alpha_\eps+\beta_\eps\ld\frac{1}{\confBase}\right)^k
	\]
	be its sample complexity with $\alpha_\eps,\beta_\eps \geq 0$. In the following, we want to compare the runtime of the~\parallelSchemeName~for achieving an $(\eps,\conf)$-guarantee to the runtime of the application of the base learning algorithm for achieving the same $(\eps,\conf)$-guarantee. 
	
	To achieve an $(\eps,\conf)$-guarantee, the~\parallelSchemeName~with parameter $h\in\N$ requires $\samplesize=nr^h$ examples (i.e., with $r^h$ processing units), where $n$ denotes the size of the data subset available to each parallel instance of the base learning algorithm. Since $\conf=(r\confBase)^{2^h}$, each base learning algorithm needs to achieve an $(\eps,\confBase)$-guarantee and thus requires at least
	\begin{equation}
	n=\left\lceil\samplesize_0^\algo(\eps,\confBase)\right\rceil\leq\left(\alpha_\eps+\beta_\eps\ld\frac{1}{\confBase}\right)^k+1
	\label{eq:proof:thm5:basen}
	\end{equation}
	examples. 
	The application of the base learning algorithm requires at least (cf. Equation~\ref{eq:sampleSizeofSequential})
	\begin{equation}
	M = \left\lceil\samplesize^\algo_0(\eps,\conf)\right\rceil =\left\lceil\left(\alpha_{\eps}+2^h\cdot\beta_{\eps}\ld\frac{1}{r\confBase}\right)^{k}\right\rceil=\left\lceil\left(\alpha_{\eps}+2^h\left(\beta_{\eps}\ld\frac{1}{\confBase}-\beta_\eps\ld r\right)\right)^{k}\right\rceil\enspace .
	\label{eq:proof:thm5:M}
	\end{equation}
	Solving Equation~\ref{eq:proof:thm5:basen} for $\beta_\eps\ld\sfrac{1}{\confBase}$ yields
	\[
	\beta_{\eps}\ld\frac{1}{\confBase} \leq (n-1)^{\frac{1}{k}}-\alpha_\eps\enspace .
	\]
	By inserting this into Equation~\ref{eq:proof:thm5:M} we obtain
	\begin{equation}
	M \geq \left\lceil\left(\alpha_\eps + 2^h\left((n-1)^\frac{1}{k}-\alpha_\eps-\beta_\eps\ld r\right)\right)^k\right\rceil \in\bigo\left(2^h\left(n - \ld r\right)\right)\enspace .
	\label{eq:proof:thm5:samplesizeOfAlgo}
	\end{equation}
	In the following, we show that for the choice of
	\begin{equation}
	h=\left\lceil\frac{1}{k}\left(\ld M - \ld\ld M\right)\right\rceil\enspace ,
	\label{eq:proof:thm5:h}
	\end{equation}
	the runtime of the~\parallelSchemeName~is polylogarithmic in $M$, i.e., polylogarithmic in the number of examples the base learning algorithm requires to achieve the same $(\eps,\conf)$-guarantee. For that, the~\parallelSchemeName~requires quasi-polynomially many processors in $M$. Note that the~\parallelSchemeName~processes $\samplesize\geq M$ many samples to achieve that $(\eps,\conf)$-guarantee, which is more than the base learning algorithm requires by a factor in $\bigo\left(\sfrac{r^h}{2^{hk}}\right)$.
	
	Thus, we need to express the runtime of the~\parallelSchemeName, that is,
	\[
	\runtime_{\radonAlgo,h}(\samplesize)=\runtime_\algo\left(\frac{\samplesize}{r^h}\right) + r^3\log_r r^h=\runtime_\algo\left(n\right) + r^3\log_r r^h\enspace ,
	\]
	in terms of $M$ instead of $\samplesize$. First, we express $n$ in terms of $M$, by solving Equation~\ref{eq:proof:thm5:samplesizeOfAlgo} for $n$ which yields
	\begin{equation}
	n \leq \left(\left(\alpha_\eps\left(1-\frac{1}{2^h}\right)+\beta_\eps\ld r + \frac{1}{2^h}M^{\frac{1}{k}}\right)^k+1\right) \in\bigo\left(\log^k_2 r + \frac{1}{2^{hk}}M\right)\enspace .
	\label{eq:proof:thm5:n}
	\end{equation}
	Since $\algo$ is efficient, $\runtime_{\algo}(n)\in\bigo(n^\kappa)$ and thus the runtime of the~\parallelSchemeName~in terms of $M$, denoted $\runtime_{\radonAlgo}^M$, is
	\[
	\runtime_{\radonAlgo}^M = \runtime_\algo\left(n\right) + r^3\log_r r^h \in \bigo\left(\left(\log^k_2 r + \frac{1}{2^{hk}}M\right)^\kappa+r^3\ld r^h\right)\enspace .
	\]
	Inserting $h$ as in Equation~\ref{eq:proof:thm5:h} yields
	\begin{equation*}
	\begin{split}
	\left(\log^k_2 r + \frac{1}{2^{hk}}M\right)^\kappa+r^3\ld\frac{M}{\ld M}=&\left(\log^k_2 r + \frac{M}{2^{k\frac{1}{k}\ld\frac{M}{\ld M}}}\right)^\kappa+r^3\ld\frac{M}{\ld M}\\
	=&\left(\log^k_2 r + \frac{M}{\frac{M}{\ld M}}\right)^\kappa+r^3\ld\frac{M}{\ld M}\\
	=&\left(\log^k_2 r + \ld M\right)^\kappa+r^3\ld\frac{M}{\ld M}\enspace .
	\end{split}
	\end{equation*}
	This shows that
	\[
	\runtime_\radonAlgo^M\in\bigo\left(\ld^\kappa M+\ld^{k\kappa}r + r^3\ld M\right)\enspace .
	\]
	Thus, the runtime of the~\parallelSchemeName~to achieve an $(\eps,\conf)$-guarantee in terms of $M$ (i.e., the number of samples required by the base learning algorithm to achieve that guarantee) is in $\bigo\left(\ld^\kappa M+\ld^{k\kappa}r + r^3\ld M\right)$ and therefore polylogarithmic in $M$.
	
	We now determine the number of processing units $c=r^h$ in terms of $M$. For that, observe that $h$ as in Equation~\ref{eq:proof:thm5:h} can be expressed as 
	\[
	h={\left\lceil\frac{1}{k}\left(\ld M - \ld\ld M\right)\right\rceil}= {\left\lceil\frac{1}{k}\left(\ld\frac{M}{\ld M}\right)\right\rceil}={\left\lceil\frac{\ld r}{k}\log_r\frac{M}{\ld M}\right\rceil}\enspace .
	\]
	With this the number of processing units is
	\begin{equation*}
	\begin{split}
	\cores = r^h 
	\in\bigo\left(M^{\ld r}\right)\enspace .
	\end{split}
	\end{equation*}
	%
\end{proof}
As mentioned in Section~\ref{se:rt}, for the~\parallelSchemeName~to achieve an $(\eps,\conf)$-guarantee each instance of its base learning algorithm has to achieve $\confBase \leq \sfrac{1}{2r}$. Thus, the sample size with respect to $M$ has to be large enough so that each base learner achieves this minimum $\confBase$. Similar to the proof of Theorem~\ref{prop:polylogRuntime}, we can express this minimum sample size in terms of $M$: The base learning algorithm achieves $\confBase \leq \sfrac{1}{2r}$ for $M\geq 2^{k\beta_\eps(\alpha_\eps + 1)}$. This can be shown by first observing that Equation~\ref{eq:proof:thm5:M} implies that for each instance of the base learning algorithm to achieve $\confBase \leq \sfrac{1}{2r}$ it is required that
\begin{equation}
M \geq \left(\alpha_{\eps}+2^h\cdot\beta_{\eps}\ld\frac{1}{r\frac{1}{2r}}\right)^{k}= \left(\alpha_{\eps}+2^h\beta_{\eps}\right)^{k}=\left(\alpha_{\eps}+\left(\frac{M}{\ld M}\right)^{\frac{1}{k}}\beta_{\eps}\right)^{k}\enspace .
\label{eq:proof:thm5:intermediateM}
\end{equation}
This holds for $M\geq 2^{k\beta_\eps(\alpha_\eps + 1)}\geq2^{k\beta_\eps}$, since
\begin{equation*}
\begin{split}
M&\underbrace{\geq}_{\ld M \geq \beta_{\eps}^k\left(\alpha_\eps+1\right)^k} \frac{M}{\ld M}\beta_{\eps}^k\left(\alpha_\eps+1\right)^k\\
&\underbrace{\geq}_{\left(\frac{M}{\ld M}\right)^{\frac{1}{k}}\beta_{\eps} \geq 1}\left(\left(\frac{M}{\ld M}\right)^{\frac{1}{k}}\beta_{\eps}\left(\frac{\alpha_\eps}{\left(\frac{M}{\ld M}\right)^{\frac{1}{k}}\beta_{\eps}}+1\right)\right)^k=\left(\alpha_{\eps}+\left(\frac{M}{\ld M}\right)^{\frac{1}{k}}\beta_{\eps}\right)^{k} \enspace .
\end{split}
\end{equation*}
After having proven that the~\parallelSchemeName~has polylogarithmic runtime on quasi-polynomially many processors, in the following section we analyse the speed-up over the base learning algorithm which achieves the same $(\eps,\conf)$-guarantee.

\subsection{Analysis of the Speed-Up of the~\parallelSchemeName}
\label{sec:appendix:proofPropRuntime}
In this section, we analyse the speed-up of the~\parallelSchemeName~over the execution of the base learning algorithm when both achieve the same $(\eps,\conf)$-guarantee, as well as its inefficiency~\citep{kruskal1990complexity} and its data inefficiency, i.e., how much more data the~\parallelSchemeName~requires compared to the base learning algorithm which achieves the same $(\eps,\conf)$-guarantee.
For that, recall that the sample complexity of the base learning algorithm for a given $\eps >0$, $0<\conf < 1$ is
\[
\samplesize_0^\algo(\eps,\conf) = \left(\alpha_{\eps}+\beta_{\eps}\log_2\frac{1}{\conf}\right)^{k} \enspace .
\]
We assume that  $\alpha_\eps\in\Theta(\eps^{-1})$ and $\beta_\eps\in\Theta(\eps^{-1})$ (see for example Lemma~\ref{lm:finiteVC} and Lemma~\ref{lm:finiteRademacher}).
Following the notion of~\citet{hanneke2016optimal} the sample complexity can then be expressed as 
\begin{equation}
\samplesize_0^\algo(\eps,\conf)\in\Theta\left(\left(\frac{1}{\eps} + \frac{1}{\eps}\ld\frac{1}{\conf}\right)^k\right)=\Theta\left(\left(\frac{1}{\eps}\ld\frac{1}{\conf}\right)^k\right) \enspace.
\label{eq:app:sampleSizeTheta}
\end{equation}
%
Similar to~\citet{kruskal1990complexity}, we assume the base algorithm to have a runtime polynomial in $\samplesize$, i.e., 
\begin{equation}
\begin{split}
\runtime_{\algo}\in\Theta\left(\samplesize^\kappa\right)=\Theta\left(\left(\frac{1}{\eps}\ld\frac{1}{\conf}\right)^{k\kappa}\right)\enspace .
\end{split}
\label{eq:landauRuntimeAlgo}
\end{equation}
The~\parallelSchemeName~runs $\algo$ in parallel on $c$ processors to obtain $r^h$ weak hypotheses with $(\eps,\confBase)$-guarantee. It then combines the obtained solutions $h$ times---level-wise in parallel---calculating the Radon point (which takes time $r^3$). 
In this paper we assume the number of available processors to be abundant and thus set $c=r^h$. With this, the runtime of the~\parallelSchemeName~is
\begin{equation}
\begin{split}
\runtime_{\radonAlgo}\in\Theta\left(  \left(\frac{1}{\eps}\ld\frac{1}{\confBase}\right)^{k\kappa} + hr^3\right) \enspace .
\end{split}
\label{eq:optimalRuntimeRadonAlgo}
\end{equation}
We now provide an analysis on the speed-up for $c=r^h$ and arbitrary $h\in\N$.
\begin{prop}
	Given a  polynomial time consistent regularised risk minimisation algorithm $\algo$ using a hypothesis space with finite Radon number $r\in\N$ and runtime as in Equation~\ref{eq:landauRuntimeAlgo}, the~\parallelSchemeName~run with parameter $h\in\N$ on $r^h$ processors. Then, the ratio of the runtime of the base learner over the runtime of the~\parallelSchemeName, denoted the speed-up~\citep{kruskal1990complexity}
	\[
	\frac{\runtime_{\algo}}{\runtime_{\radonAlgo}}\enspace ,
	\]
	is in
	\[
	\Theta\left(\frac{2^{hk\kappa}}{1+\frac{hr^3}{\left(\frac{1}{\eps}\ld\frac{1}{\confBase}\right)^{k\kappa}}}\right)\enspace .
	\]
\end{prop}
\begin{proof}
	In order to achieve an  $(\eps,\conf)$-guarantee, the~\parallelSchemeName~runs $r^h$ parallel instances of the the base learning algorithm on $\basesamplesize=\lceil \samplesize^\algo_0(\eps,\confBase)\rceil$ examples with $\confBase\leq\sfrac{1}{2r}$ so that $\conf=(r\confBase)^{2^h}$.
	To achieve the same $(\eps,\conf)$-guarantee, the base learning algorithm requires 
	\[
	M=\left\lceil \samplesize^\algo_0(\eps, \conf)\right\rceil=\left\lceil\left(2^h\cdot\frac{1}{\eps}\ld\frac{1}{r\confBase}\right)^{k}\right\rceil\in\Theta\left(\left(2^{h}\frac{1}{\eps}\ld\frac{1}{r\confBase}\right)^k\right) = \Theta\left(\left(2^{h}\frac{1}{\eps}\ld\frac{1}{\confBase}\right)^k\right)
	\]
	many examples. The last step follows from the fact that, since $\confBase \leq \sfrac{1}{2r}$, we have $\sfrac{1}{r\confBase}\geq\sfrac{2r}{r}\geq r$ and thus
	\begin{equation*}
	\begin{split}
	&\ld\frac{1}{r\confBase}\leq \ld\frac{1}{r\confBase} + \ld r = \ld\frac{1}{\confBase} \leq 2 \ld\frac{1}{r\confBase}\\
	\Rightarrow & \ld \frac{1}{\confBase} \in \Theta\left(\ld\frac{1}{r\confBase}\right)\ \Leftrightarrow\ \ld \frac{1}{r\confBase} \in \Theta\left(\ld\frac{1}{\confBase}\right)\enspace ,
	\end{split}
	\end{equation*}
	To achieve the $(\eps,\conf)$-guarantee, the base learning algorithm has a runtime of
	\[
	\runtime_{\algo}\ \in\ \Theta\left(M^\kappa\right)\ =\ \Theta\left(\left(2^h\frac{1}{\eps}\ld\frac{1}{\confBase}\right)^{k\kappa}\right)\enspace .
	\]
	Using $\runtime_{\radonAlgo}$ from Equation~\ref{eq:optimalRuntimeRadonAlgo}, we get that 
	\begin{equation*}
	\begin{split}
	\frac{\runtime_{\radonAlgo}}{\runtime_{\algo}} \in &\Theta\left(\frac{\left(\frac{1}{\eps}\ld\frac{1}{\confBase}\right)^{k\kappa} + hr^3}{\left(2^h\frac{1}{\eps}\ld\frac{1}{\confBase}\right)^{k\kappa}}\right) = \Theta\left(\frac{1}{2^{hk\kappa}}\left(1+\frac{hr^3}{\left(\frac{1}{\eps}\ld\frac{1}{\confBase}\right)^{k\kappa}}\right)\right)
	\end{split}
	\end{equation*}
	The speed-up then is
	\[
	\frac{\runtime_{\algo}}{\runtime_{\radonAlgo}}\in \Theta\left(\frac{2^{hk\kappa}}{1+\frac{hr^3}{\left(\frac{1}{\eps}\ld\frac{1}{\confBase}\right)^{k\kappa}}}\right)\enspace .
	\]
\end{proof}

Note that the runtime of the~\parallelSchemeName~for the case that $1\leq c\leq r^h$ is given by
\begin{equation*}
\begin{split}
\runtime_{\radonAlgo}\in\Theta\left( \frac{r^h}{c}\left( \left(\frac{1}{\eps}\ld\frac{1}{\confBase}\right)^{k\kappa}\right) + r^3\sum_{i=1}^h\left\lceil\frac{r^i}{c}\right\rceil\right) \enspace .
\end{split}
\end{equation*}
In this case, the speed-up is lower by a factor of $\sfrac{r^h}{c}$.

In the following, we analyse the inefficiency~\citep{kruskal1990complexity} of the~\parallelSchemeName, i.e., the ratio between the total number of operations executed by all processors and the work of the base learning algorithm.
\begin{prop}
	The~\parallelSchemeName~with a consistent and efficient regularised risk minimisation algorithm on a hypothesis space with finite Radon number has quasi-polynomial inefficiency if the Radon number is upper bounded by a function polylogarithmic in the sample complexity of the efficient regularised risk minimisation algorithm.
\end{prop}
\begin{proof}
	Let $\algo$ be a consistent and efficient regularised risk minimisation algorithm on a hypothesis space with finite Radon number $r\in\N$. Since $\algo$ is efficient, its runtime $\runtime_{\algo}(M)$ is in $\bigo(M^\kappa)$.
	From the proof of Theorem~\ref{prop:polylogRuntime} follows that, when choosing $h={\left\lceil\frac{1}{k}\left(\ld M - \ld\ld M\right)\right\rceil}$ the~\parallelSchemeName~has a runtime of
	$\runtime_{\radonAlgo}(M)\in\bigo\left(\ld^\kappa M + \ld^{k\kappa}r + r^3\ld M\right)$ using $c\in\bigo\left(M^{\ld r}\right)$ processing units. The inefficiency of the~\parallelSchemeName~then is
	\begin{equation*}
	\begin{split}
	\frac{c \cdot \runtime_{\radonAlgo}(M)}{\runtime_{\algo}(M)}\in\bigo\left(\frac{M^{\ld r}\left(\ld^\kappa M + \ld^{k\kappa}r + r^3\ld M\right)}{M^\kappa}\right) \in \bigo\left(M^{(\ld r) - \kappa}\ld^\kappa M\right)=\bigo\left(M^{\ld r}\right)\enspace .
	\end{split}
	\end{equation*}
	Thus, the inefficiency of the~\parallelSchemeName~is quasi-polynomially bounded or, for short, it has quasi-polynomial inefficiency.
\end{proof}
In order to achieve the same $(\eps,\conf)$-guarantee as the base learning algorithm, the~\parallelSchemeName~requires more data. In the following, we analyse the data inefficiency $\sfrac{\samplesize_\radonAlgo(\eps,\conf)}{\samplesize_\algo(\eps,\conf)}$, i.e., the ratio of the data required by the~\parallelSchemeName~over the data required by the base learning algorithm.
\begin{prop}
	The~\parallelSchemeName~with a consistent and efficient regularised risk minimisation algorithm $\algo$ with sample complexity $\samplesize_\algo(\eps, \conf)$
	on a hypothesis space with finite Radon number $r\in\N$ has a data inefficiency in
	\[
	\Theta\left(\left(\frac{M}{\ld M}\right)^{\frac{\ld r}{k}}\right)\enspace ,
	\]
	where $M=\lceil\samplesize_\algo(\eps, \conf)\rceil$.
\end{prop}
\begin{proof}
	We assume the sample complexity can be expressed as in Equation~\ref{eq:app:sampleSizeTheta}. For $\conf=(r\confBase)^{2^h}$ we have that
	\begin{equation*}
	\begin{split}
	\samplesize_\radonAlgo(\eps, \conf)=&r^h\samplesize_\algo(\eps,\confBase)\in\Theta\left(r^h\left(\frac{1}{\eps}\ld\frac{1}{\confBase}\right)^k\right)\\
	=&\Theta\left(r^h\left(\frac{1}{2^h}\frac{1}{\eps}\ld\frac{1}{\conf}\right)^k\right)=\Theta\left(\frac{r^h}{2^{hk}}\left(\frac{1}{\eps}\ld\frac{1}{\conf}\right)^k\right)\\
	=&\Theta\left(\frac{r^h}{2^{hk}}\samplesize_\algo(\eps,\conf)\right)=\Theta\left(\frac{r^h}{2^{hk}}M\right)\enspace .
	\end{split}
	\end{equation*}
	Thus, the data inefficiency is in
	\[
	\Theta\left(\frac{r^h}{2^{hk}}\right)\enspace .
	\]
	Choosing $h=\lceil k^{-1}(\ld M - \ld\ld M)\rceil$ as in the proof of Theorem~\ref{prop:polylogRuntime}, this is in
	\[
	\Theta\left(\frac{r^{\frac{1}{k}\ld r\log_r\frac{M}{\ld M}}}{2^{k\frac{1}{k}\ld\frac{M}{\ld M}}}\right) = \Theta\left(\frac{\left(\frac{M}{\ld M}\right)^{\frac{\ld r}{k}}}{\frac{M}{\ld M}}\right) = \Theta\left(\left(\frac{M}{\ld M}\right)^{\frac{\ld r}{k}}\right)
	\]
\end{proof}

\section{Experiments}
\label{appendix:addExps}
This section provides additional details on the experiments conducted. All experiments are performed on a Spark cluster with a master node, $5$ worker nodes, $25$ processors and $64$GB of RAM per node.
The~\parallelSchemeName~is applied with parameter $h=1$ and with the maximal $h$ for a given dataset: Recall, that the number of iterations $h$ is limited by the dataset size (i.e., number of instances) and the Radon number of the hypothesis space, since the dataset is partitioned into $r^h$ parts of size $n$. Thus, given a data set of size $\samplesize$, the maximal $h$ is given by
\[
h_{\text{max}}=\left\lfloor\log_r\frac{\samplesize}{\basesamplesize_{\text{min}}}\right\rfloor\enspace ,
\]
where $\basesamplesize_{\text{min}}$ denotes the minimum size of the local subset of data that each instance of the base learner is executed on. The experiments have been carried out with $\basesamplesize_{\text{min}}=100$.
As discussed in Section~\ref{sec:discussion}, if $r^h$ is larger than the actual number of processing units, some instances of the base learner are executed sequentially.

As base learning algorithms we use the WEKA~\citep{witten2016data} implementation of Stochastic Gradient Descent (\emph{WekaSGD}), and Logistic Regression (\emph{WekaLogReg}), as well as a the Scikit-learn~\cite{scikit-learn} implementation of the linear support vector machine (\emph{LinearSVM}) with pyspark. The paralellisation of a base learner using the~\parallelSchemeName~is denoted~\emph{PRM(h=?)[$<$base learner$>$]}. 

We compare the~\parallelSchemeName~to the natural baseline of aggregating hypotheses by calculating their average, denoted averaging-at-the-end~({\emph{Avg(h=?)[$<$base learner$>$]}}). Given a parameter $h\in\N$, averaging-at-the-end executes the base learning algorithm on $r^h$ subsets of the data, i.e., on the same number of subsets as the~\parallelSchemeName. Accordingly, the runtime for obtaining the set of hypotheses is similar, but the time for aggregating the models is shorter, since averaging is less computationally expensive than calculating the Radon point.
\begin{table}[ht]
	\begin{center}
		\begin{tabular}		
			{l r r l}
			\textbf{Name}	 		& \textbf{Instances} 	& \textbf{Dimensions}	& \textbf{Output} \\[0.15cm]	
			click\_prediction		& $1\,496\,391$ 	& 11 	& $\outputspace=\{-1,1\}$		\\[0.15cm]	
			poker	 				& $1\,025\,010$ 	& 10	& $\outputspace=\{-1,1\}$ 		\\[0.15cm]	
			SUSY 					& $5\,000\,000$ 	& 18 	& $\outputspace=\{-1,1\}$		\\[0.15cm]	
			Stagger1 				& $1\,000\,000$ 	& 9 	& $\outputspace=\{-1,1\}$ 		\\[0.15cm]	
			HIGGS 					& $11\,000\,000$ 	& 28	& $\outputspace=\{-1,1\}$		\\[0.15cm]	
			SEA\_50					& $1\,000\,000$ 	& 3 	& $\outputspace=\{-1,1\}$ 		\\[0.15cm]									
			codrna 					& $488\,565$ 	& 8 	& $\outputspace=\{-1,1\}$ 		\\[0.15cm]			
			CASP9					& $31\,993\,555$ 	& 631 	& $\outputspace=\{-1,1\}$ 		\\[0.15cm]
			wikidata				& $19\,254\,100$ 	& 2331	& $\outputspace=\{-1,1\}$ 		\\[0.15cm]
			20\_newsgroups			& $399\,940$ 	& 1002	& $\outputspace=\{-1,1\}$ 		\\[0.15cm]
			YearPredictionMSD		& $515\,345$ 	& 90 	& $\outputspace\subseteq\R$		\\[0.15cm]
			drift					& $13\,991$ 	& 90 	& $\outputspace=\{1,\dots,89\}$	\\[0.15cm]	
			spoken-arabic-digit		& $263\,256$ 	& 15 	& $\outputspace=\{1,\dots,10\}$	\\[0.15cm]																													
			\vspace{0.1cm}
		\end{tabular}
		\caption{Description of the datasets used in our experiments.}
		\label{tbl:datasets}
	\end{center}
\end{table}

We also compare the~\parallelSchemeName~to parallel machine learning algorithms from the Spark machine learning library (MLlib): SparkMLLibLogisticRegressionWithLBFGS (\emph{SparkLogRegwLBFGS}), SparkMLLibLogisticRegressionWithSGD (\emph{SparkLogRegwSGD}), SparkMLLibSVMWithSGD (\emph{SparkSVMwSGD}), and SparkMLLogisticRegression (\emph{SparkLogReg}). 

The properties of the datasets used in the empirical evaluation are presented in Table~\ref{tbl:datasets}. Datasets have been acquired from OpenML~\citep{OpenML2013}, the UCI machine learning repository~\citep{Lichman:2013}, and Big Data competition of the ECDBL'14 workshop\footnote{Big Data Competition 2014: \url{http://cruncher.ncl.ac.uk/bdcomp/}}. Experiments on moderately sized datasets---on which we compare the~\parallelSchemeName~to the base learning algorithms executed on the entire dataset are conducted on the datasets click\_prediction, poker, SUSY, Stagger1, SEA\_50, and codrna. The comparison of~\parallelSchemeName~and Spark MLlib learners is executed on the datasets CASP9, HIGGS, wikidata, 20\_newsgroups, and SUSY. The regression experiment is conducted using the YearPredictionMSD dataset, multiclass-prediction experiments using the drift, and spoken-arabic-digit datasets.

In the following, we provide more details on the experiments presented in Figures~\ref{fig:expBoxPlotsRuntimeAUCCentral}, \ref{fig:radonVsAverage}, and \ref{fig:expBoxPlotsRuntimeAUCSpark} in Section~\ref{sec:experiments}. In particular, we analyse the trade-off between training time and AUC per dataset. 
\begin{figure*}[t]
	\centering	
	\includegraphics[width=13cm]{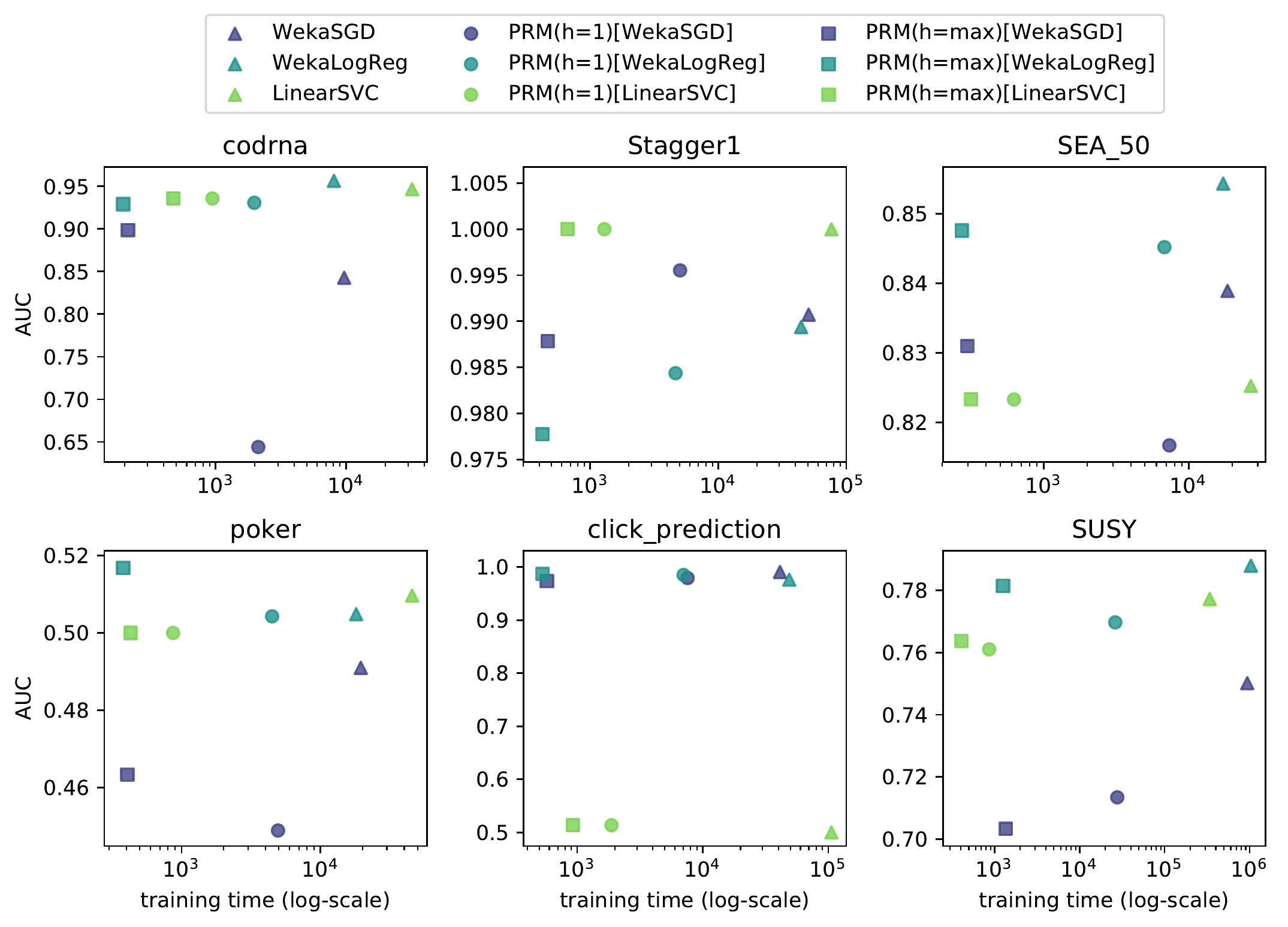}
	\caption{AUC vs. training time for base learning algorithms and their parallelisation with the~\parallelSchemeName~per dataset from the same experiment as in Figure~\ref{fig:expBoxPlotsRuntimeAUCCentral}.}
	\label{fig:singleVsRadonRtVsAucPerDataset}
\end{figure*}
\begin{figure*}[ht]
	\centering	
	\includegraphics[width=13cm]{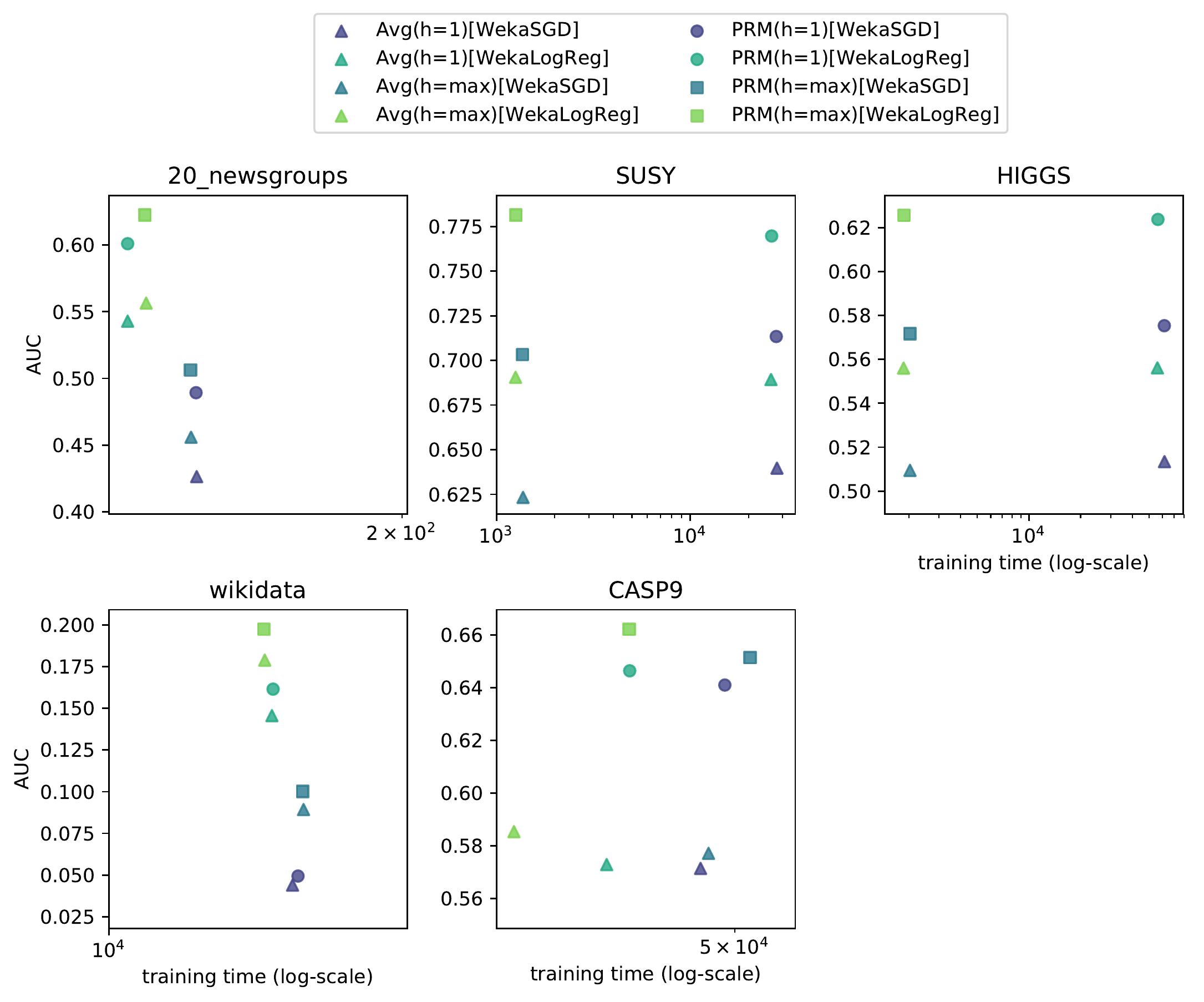}
	\caption{AUC vs. training time for the parallelisation of base learning algorithms using the averaging-at-the-end baseline (\emph{Avg}) and the~\parallelSchemeName~per dataset from the same experiment as in Figure~\ref{fig:radonVsAverage}.}
	\label{fig:avgVsRadonRtVsAucPerDataset}
\end{figure*}

Figure~\ref{fig:singleVsRadonRtVsAucPerDataset} shows the trade-off between training time and AUC for base learning algorithms and their parallelisation using the~\parallelSchemeName. It confirms that the training time for the~\parallelSchemeName~is orders of magnitude smaller than the base learning algorithms on all datasets. Moreover, the training time is substantially smaller for the~\parallelSchemeName~with maximal height ($h=max$), compared to the parameter $h=1$. In terms of AUC, the performance of the parallelisation is comparable to the base learner for WekaLogReg and LinearSVC on all datasets. For the base learner WekaSGD, the predictive performance of the parallelisation with the~\parallelSchemeName~is comparable on all datasets but codrna. There, the~\parallelSchemeName~with parameter $h=1$ has substantially lower AUC, while the parallelisation with $h=max$ has substantially higher AUC than the base learning algorithm executed on the entire dataset. 

The comparison of the~\parallelSchemeName~to the averaging-at-the-end baseline in Figure~\ref{fig:avgVsRadonRtVsAucPerDataset} confirms the findings of Section~\ref{sec:experiments}, i.e., the~\parallelSchemeName~achieves a substantially higher AUC with only slightly higher runtime. Comparing the~\parallelSchemeName~to the Spark MLlib learning algorithms in Figure~\ref{fig:sparkVsRadonRtVsAucPerDataset} indicates that the~\parallelSchemeName~is always favourable in terms of training time. However, in terms of AUC the results are mixed: For the base learner WekaLogReg, its parallelisation is always among the best in terms of AUC. The parallelisation of WekaSGD, however, has worse performance than the Spark learners on $2$ out of $5$ datasets. It also confirms that for the datasets SUSY and HIGGS, the runtime of the~\parallelSchemeName~with $h=1$ is substantially larger than for $h=max$. Thus, for the best performance in terms of runtime and AUC, the height should be maximal.
\begin{figure*}[h!]
	\centering	
	\includegraphics[width=13cm]{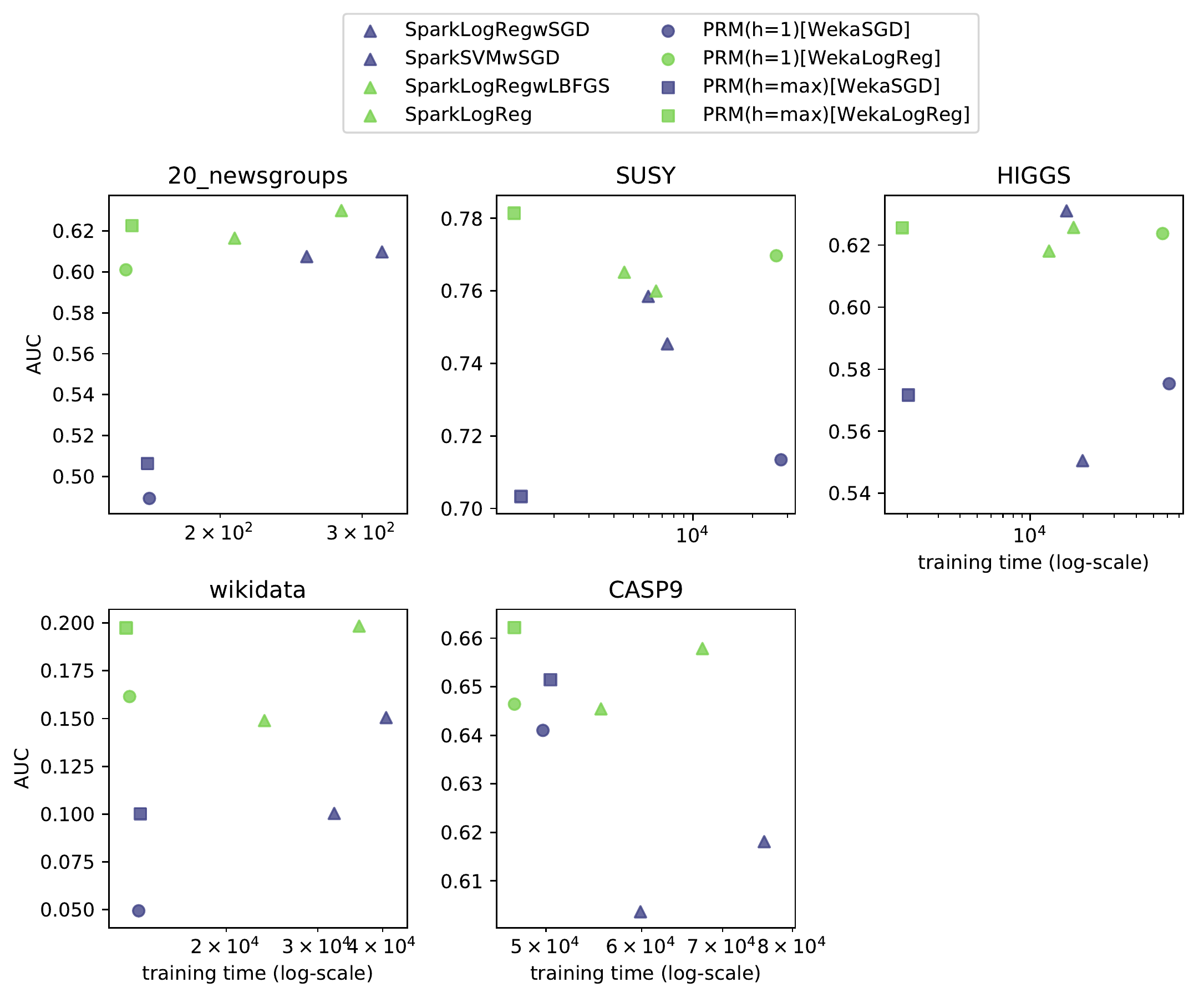}
	\caption{AUC vs. training time for Spark learners and parallelisations of comparable base learning algorithms with the~\parallelSchemeName~per dataset from the same experiment as in Figure~\ref{fig:expBoxPlotsRuntimeAUCSpark}.}
	\label{fig:sparkVsRadonRtVsAucPerDataset}
\end{figure*}

In order to investigate the results depicted in Figure~\ref{fig:sparkVsRadonRtVsAucPerDataset} more closely, we provide the training times and AUCs in detail in Table~\ref{tbl:sparkExpResults}. 
As mentioned above, the~\parallelSchemeName~using WekaLogReg as base learner has better runtime than all Spark algorithms.  At the same time, this version of the~\parallelSchemeName~outperforms the Spark algorithms in terms of AUC on all datasets but 20\_newsgroups---there it is $2.2\%$ worse than the best Spark algorithm. In particular, on the largest dataset in the experiments---the CASP9 dataset with $32$ million instances and $631$ features---the~\parallelSchemeName~is $15\%$ faster and $2.6\%$ better in terms of AUC than the best Spark algorithm. 

\begin{figure*}[ht]
	\centering	
	\subfigure[]{\label{fig:expRuntimeOfRepartition}\includegraphics[width=7cm]{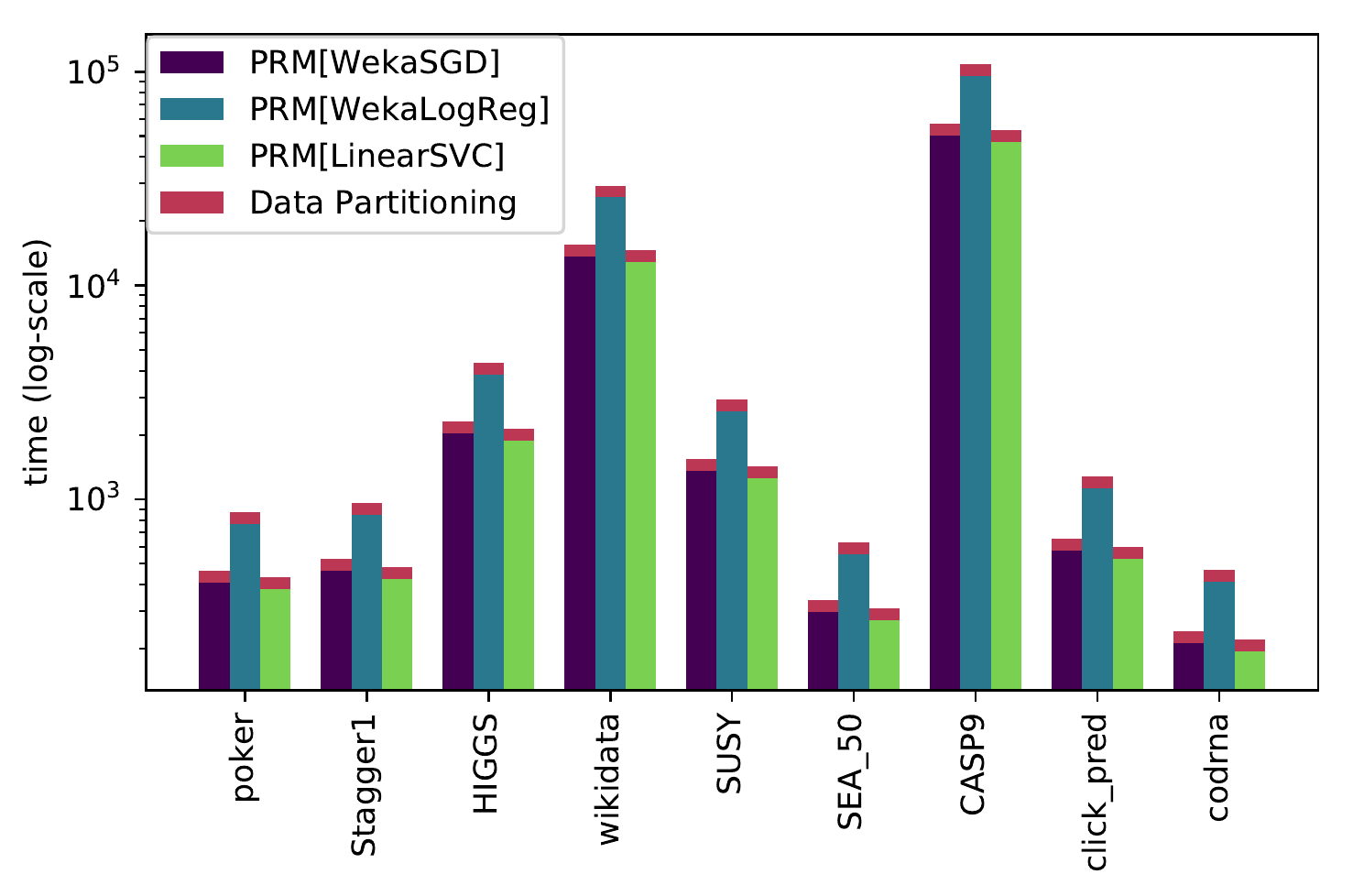}}\hfill%
	\subfigure[]{\label{fig:expSparkWithDataSplit}\includegraphics[height=7cm]{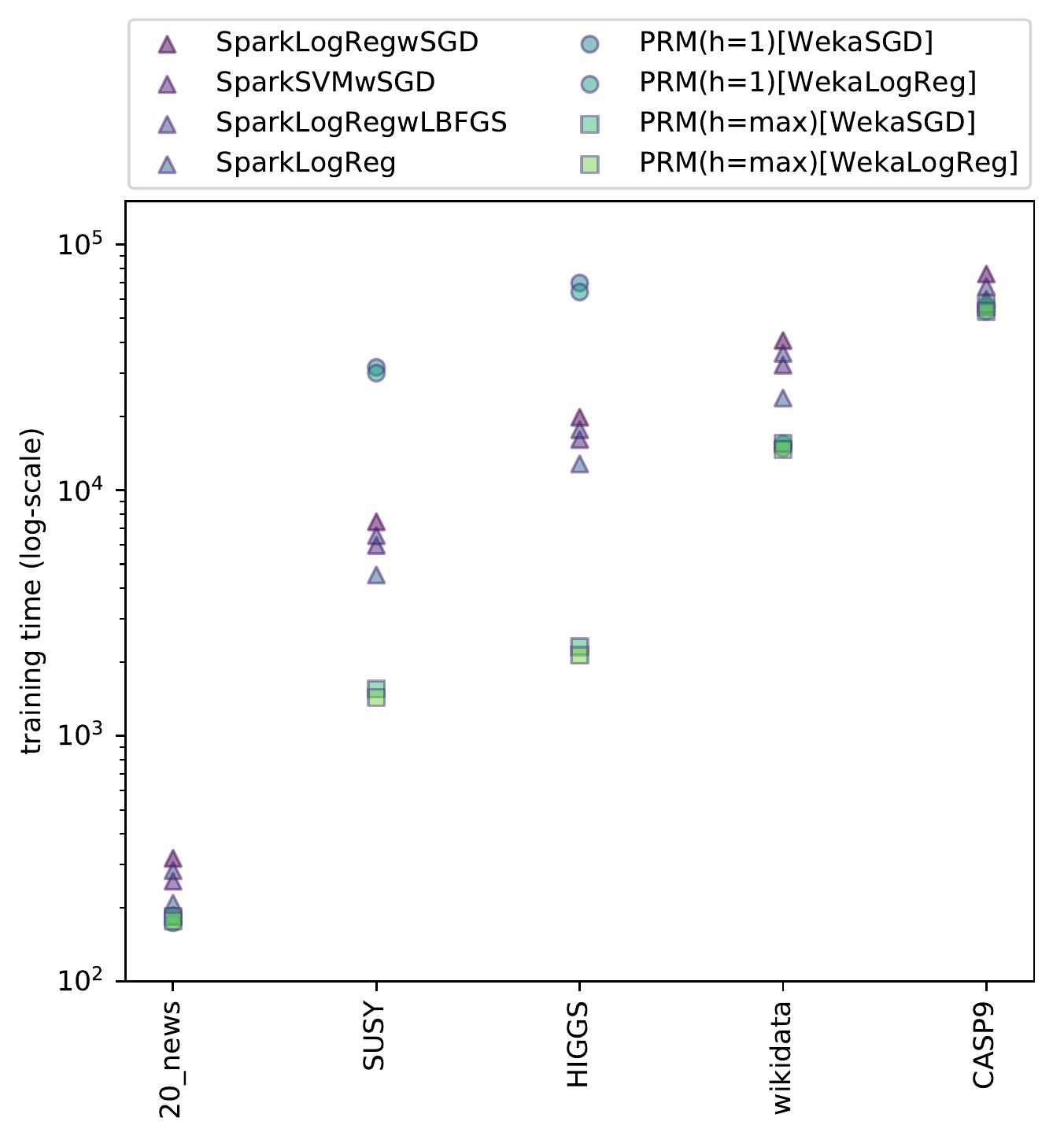}}\hfill%
	\caption{
		\subref{fig:expRuntimeOfRepartition} Runtime of the~\parallelSchemeName~together with the time required for repartitioning the data to fit the parallelisation scheme.~		
		\subref{fig:expSparkWithDataSplit} Runtime and AUC of several Spark machine learning library algorithms and the~\parallelSchemeName~including the time required for repartitioning the data before training.
	}
\end{figure*}
Note that for HIGGS and SUSY, the~\parallelSchemeName~with $h=1$ is an order of magnitude slower than with $h=max$ as well as the Spark algorithms. This follows from the low degree of parallelisation, since for $h=1$ only $20$ (for SUSY), respectively $30$ (for HIGGS) hypotheses have to be generated. Thus, only $20$, or $30$ of the $150$ available processors are used in parallel. At the same time, the amount of data each processor has to process is orders of magnitude larger than for $h=max$. 
\begin{table}[b]
	\tiny
	\begin{center}
		\begin{tabularx}		
			{\columnwidth}{XYYYYYYYY}
			\textbf{Dataset}&	\multicolumn{8}{c}{\textbf{Runtime}} \\[0.15cm]	
			&	\textbf{SparkLogReg\newline wSGD}	&	\textbf{SparkSVM\newline wSGD}	&	\textbf{PRM(h=1)\newline [WekaSGD]}	& \textbf{PRM(h=max)\newline [WekaSGD]}	&	\textbf{SparkLogReg\newline wLBFGS}	&	\textbf{SparkLogReg}	&	\textbf{PRM(h=1)\newline [WekaLogReg]}		&	\textbf{PRM(h=max)\newline [WekaLogReg]}\\[0.15cm]	
			20\_newsgroups	&	$317.7$ 	&	$256.2$				&	$163.4$ 	&	$162.5$ 	&	$282.9$ 		&	$208.5$ 	&	$\mathbf{152.8}$	&	$155.4$ \\
			SUSY			&	$7\,439.5$ 	&	$5\,961.8$ 			&	$27\,781.6$ &	$1\,363.7$ 	&	$6\,526.3$ 		&	$4\,516.8$ 	&	$26\,299.6$ 		&	$\mathbf{1259.7}$ \\
			HIGGS			&	$19\,815.1$	&	$16\,071.9$ 		&	$61\,429.5$ &	$2\,029.7$ 	&	$17\,617.4$		&	$12\,783.6$	&	$56\,394.2$ 		&	$\mathbf{1876.2}$ \\
			wikidata		&	$40\,645.8$	&	$32\,288.5$ 		&	$13\,575.7$ &	$13\,677.3$ &	$36\,060.1$		&	$23\,702.0$	&	$13\,039.5$ 		&	$\mathbf{12845.5}$ \\
			CASP9			&	$75\,782.4$	&	$59\,864.7$ 		&	$49\,711.5$ &	$50\,430.6$ &	$67\,367.3$ 	&	$55\,523.5$ &	$47\,085.1$ 		&	$\mathbf{47070.1}$\\[0.15cm]	
			\hline\\[0.001cm]			
			&	\multicolumn{8}{c}{\textbf{AUC}} \\[0.15cm]					
			20\_newsgroups	&	$0.6098$ 	&	$0.6075$ 			&	$0.4893$ 	&	$0.5063$ 	&	$\mathbf{0.63\phantom{1}\phantom{1}}$ &	$0.6165$ 	&	$0.601\phantom{1}$ 			&	$0.6226$\\
			SUSY			&	$0.7454$ 	&	$0.7585$ 			&	$0.7134$ 	&	$0.7033$ 	&	$0.76\phantom{1}\phantom{1}$ 			&	$0.7652$ 	&	$0.7697$ 			&	$\mathbf{0.7814}$\\
			HIGGS			&	$0.5506$ 	&	$\mathbf{0.631\phantom{1}}$ 	&	$0.5753$ 	&	$0.5717$ 	&	$0.6257$ 		&	$0.6181$ 	&	$0.6237$ 			&	$0.6256$\\
			wikidata		&	$0.1505$ 	&	$0.1004$ 			&	$0.0494$ 	&	$0.1002$ 	&	$0.1983$ 		&	$0.1489$ 	&	$0.1615$ 			&	$\mathbf{0.1974}$\\
			CASP9			&	$0.6181$ 	&	$0.6037$ 			&	$0.641$ 	&	$0.6514$ 	&	$0.6579$ 		&	$0.6454$ 	&	$0.6464$ 			&	$\mathbf{0.6622}$\\
			\vspace{0.1cm}
		\end{tabularx}
		\caption{Runtime and AUC of Spark machine learning library algorithms and the~\parallelSchemeName~using WekaSGD and WekaLogReg as base learning algorithms. The results, reported for each dataset, are the average over all folds in a $10$-fold cross-validation. These results correspond to the ones presented in Figure~\ref{fig:expBoxPlotsRuntimeAUCSpark} in Section~\ref{sec:experiments}.}
		\label{tbl:sparkExpResults}
	\end{center}
\end{table}

For the above experiments we assume that the data is already distributed over the nodes in the cluster so that it can directly be processed by the~\parallelSchemeName. 
When loading data in Spark, this data is distributed over the worker nodes in subsetsbut not necessarily in $r^h$ subsets. In Spark, distributed data is organised in partitions, where each partition corresponds to the subset of data available to one instance of the base learning algorithm. In order to apply the~\parallelSchemeName~to a dataset within the Spark framework, the data needs to re-distributed and partitioned into $r^h$ partitions which is achieved by a method called repartition.  In the experiments in Section~\ref{sec:experiments}, we assume that the data is already partitioned to make a fair comparison to the Spark learning algorithms which do not require repartitioning.
Figure~\ref{fig:expRuntimeOfRepartition} illustrates the time required for repartitioning a dataset in contrast to the runtime of the~\parallelSchemeName. Repartitioning in Spark always includes a complete shuffling of the data, requiring communication to redistribute the dataset. This is rather inefficient in our context. Nonetheless, the time required for repartitioning is small compared to the overall runtime---in the worst case it takes $14\%$ of the runtime of the~\parallelSchemeName. Still, taking into account the time for repartitioning the data shrinks the runtime advantage of the proposed scheme over the Spark algorithms. Figure~\ref{fig:expSparkWithDataSplit} shows the runtimes of the Spark algorithms compared to the~\parallelSchemeName---similar to Figure~\ref{fig:expBoxPlotsRuntimeAUCSpark} in Section~\ref{sec:experiments}---but with the time required for repartitioning the data added to the runtime of the \emph{Radon machines}. The~\parallelSchemeName~with $h=max$ remains superior to the Spark algorithms in terms of runtime.

\section{Practical Aspects}

\subsection{Radon Point Construction}
\label{sec:app:radonPointConstr}
In the following, a simple construction is given for a system of linear equations with which a Radon point of a set can be determined. In his main theorem, Radon~\cite{radon1921mengen} gives the following construction of a Radon point for a set $S=\{s_1,...,s_r\}\subseteq\R^d$. Find a non-zero solution $\lambda\in\R^{|S|}$ for the following linear equations.
\[
\sum_{i=1}^r\lambda_i s_i = (0,\dots,0)\enspace , \enspace \sum_{i=1}^r \lambda_i = 0
\]
Such a solution exists, since $|S| > d+1$ implies that $S$ is linearly dependent. Then, let $I,J$ be index sets such that for all $i\in I: \lambda_i\geq 0$ and for all $j\in J: \lambda_j < 0$. Then a Radon point is defined by
\[
\radonPoint(\lambda) = \sum_{i\in I}\frac{\lambda_i}{\Lambda}s_i = \sum_{j\in J}\frac{\lambda_j}{\Lambda}s_j\enspace ,
\]
where $\Lambda = \sum_{i\in I}\lambda_i = -\sum_{j\in J}\lambda_j$. Any solution to this linear system of equations is a Radon point. The equation system can be solved in time $r^3$. By setting the first element of $\lambda$ to one, we obtain a unique solution of the system of linear equations. Using this solution $\lambda$, we define the Radon point of a set $S$ as $\radonPoint(S)=\radonPoint(\lambda)$ in order to resolve ambiguity.

\subsection{Consistency Results for Empirical Risk Minimisation}
\label{sec:appendix:consistency}
In this section we provide some technical results on the consistency of empirical risk minimisation algorithms. 
\begin{lm}
	For consistent empirical risk minimisers with a hypothesis class of finite Vapnik-Chervonenkis (VC) dimension the sample size required to achieve an $(\eps,\conf)$-guarantee is given by $\samplesize(\conf)=(\alpha_\eps+\beta_\eps\log_2\sfrac{1}{\conf})^k$ with $\alpha_\eps = 4\ln 2\sfrac{1}{\eps^2},\beta_\eps=\sfrac{4}{\eps^2\log_2 e}$ and $k=2$.
	\label{lm:finiteVC}
\end{lm}
\begin{proof}
	For consistent empirical risk minimisers with finite VC-dimension, the confidence $1-\conf$ for a given $\samplesize$ and $\eps$ is $\conf=2\mathcal{N}(\solutionset,\samplesize)\exp(-\samplesize\eps^2/4)$~\cite{luxburg2009statistical}, where the shattering coefficient $\mathcal{N}(\solutionset,\samplesize)$ is a polynomial in $\samplesize$ for finite VC-dimension. Solving for $\samplesize$ yields that the algorithm run with 
	\[
	\samplesize \geq \frac{1}{\eps^2}\left(\ln 2 + 4\frac{1}{\log_2(e)}\log_2\frac{1}{\delta}\right)
	\]
	achieves a confidence larger or equal to the desired $1-\conf$.
\end{proof}
\begin{lm}
	For consistent empirical risk minimisers with a hypothesis class of finite Rademacher complexity the sample size required to achieve an $(\eps,\conf)$-guarantee is given by $\samplesize(\conf)=(\alpha_\eps+\beta_\eps\log_2\sfrac{1}{\conf})^k$ with $\alpha_\eps = 0,\beta_\eps=\sfrac{1}{2(\eps+2\rho)^2}$ and $k=1$, where $\rho$ denotes the Rademacher complexity.
	\label{lm:finiteRademacher}
\end{lm}
\begin{proof}
	For consistent empirical risk minimisers with a hypothesis class of finite Rademacher complexity $\rho$, a given $\conf$ and $\samplesize$ the error bound is given by $\eps=2\rho+\sqrt{\sfrac{\log_2\sfrac{1}{\delta}}{2\samplesize}}$~\cite{luxburg2009statistical}. Solving for $\samplesize$ yields the above result.
\end{proof}
\newpage
\bibliographystyle{plainnat}
\bibliography{bibliography}

\begin{thebibliography}{46}
\providecommand{\natexlab}[1]{#1}
\providecommand{\url}[1]{\texttt{#1}}
\expandafter\ifx\csname urlstyle\endcsname\relax
  \providecommand{\doi}[1]{doi: #1}\else
  \providecommand{\doi}{doi: \begingroup \urlstyle{rm}\Url}\fi

\bibitem[Arora and Barak(2009)]{arora2009computational}
Sanjeev Arora and Boaz Barak.
\newblock \emph{Computational complexity: {A} modern approach}.
\newblock Cambridge University Press, 2009.

\bibitem[Balcan et~al.(2016)Balcan, Liang, Song, Woodruff, and
  Xie]{balcan2016kernelpca}
Maria~Florina Balcan, Yingyu Liang, Le~Song, David Woodruff, and Bo~Xie.
\newblock Communication efficient distributed kernel principal component
  analysis.
\newblock In \emph{Proceedings of the 22nd ACM SIGKDD International Conference
  on Knowledge Discovery and Data Mining}, pages 725--734, 2016.

\bibitem[Bartlett and Mendelson(2003)]{bartlett2003rademacher}
Peter~L. Bartlett and Shahar Mendelson.
\newblock Rademacher and gaussian complexities: Risk bounds and structural
  results.
\newblock \emph{Journal of Machine Learning Research}, 3:\penalty0 463--482,
  2003.

\bibitem[Blum(1967)]{blum1967machine}
Manuel Blum.
\newblock A machine-independent theory of the complexity of recursive
  functions.
\newblock \emph{Journal of the ACM (JACM)}, 14\penalty0 (2):\penalty0 322--336,
  1967.

\bibitem[Blumer et~al.(1989)Blumer, Ehrenfeucht, Haussler, and
  Warmuth]{blumer1989learnability}
Anselm Blumer, Andrzej Ehrenfeucht, David Haussler, and Manfred~K Warmuth.
\newblock Learnability and the {V}apnik-{C}hervonenkis dimension.
\newblock \emph{Journal of the ACM (JACM)}, 36\penalty0 (4):\penalty0 929--965,
  1989.

\bibitem[Boyd et~al.(2011)Boyd, Parikh, Chu, Peleato, and
  Eckstein]{boyd2011distributed}
Stephen Boyd, Neal Parikh, Eric Chu, Borja Peleato, and Jonathan Eckstein.
\newblock Distributed optimization and statistical learning via the alternating
  direction method of multipliers.
\newblock \emph{Foundations and Trends{\textregistered} in Machine Learning},
  3\penalty0 (1):\penalty0 1--122, 2011.

\bibitem[Chandra and Stockmeyer(1976)]{chandra1976alternation}
Ashok~K. Chandra and Larry~J. Stockmeyer.
\newblock Alternation.
\newblock In \emph{17th Annual Symposium on Foundations of Computer Science},
  pages 98--108, 1976.

\bibitem[Clarkson et~al.(1996)Clarkson, Eppstein, Miller, Sturtivant, and
  Teng]{clarkson1996approximating}
Kenneth~L. Clarkson, David Eppstein, Gary~L. Miller, Carl Sturtivant, and
  Shang-Hua Teng.
\newblock Approximating center points with iterative {R}adon points.
\newblock \emph{International Journal of Computational Geometry \&
  Applications}, 6\penalty0 (3):\penalty0 357--377, 1996.

\bibitem[Cook(1979)]{cook1979deterministic}
Stephen~A. Cook.
\newblock Deterministic {CFL}'s are accepted simultaneously in polynomial time
  and log squared space.
\newblock In \emph{Proceedings of the eleventh annual ACM symposium on Theory
  of computing}, pages 338--345, 1979.

\bibitem[Dekel et~al.(2012)Dekel, Gilad-Bachrach, Shamir, and
  Xiao]{dekel2012optimal}
Ofer Dekel, Ran Gilad-Bachrach, Ohad Shamir, and Lin Xiao.
\newblock Optimal distributed online prediction using mini-batches.
\newblock \emph{Journal of Machine Learning Research}, 13\penalty0
  (1):\penalty0 165--202, 2012.

\bibitem[Fine and Scheinberg(2002)]{fine2002efficient}
Shai Fine and Katya Scheinberg.
\newblock Efficient svm training using low-rank kernel representations.
\newblock \emph{Journal of Machine Learning Research}, 2:\penalty0 243--264,
  2002.

\bibitem[Freund et~al.(2001)Freund, Mansour, and Schapire]{freund_why_2001}
Yoav Freund, Yishay Mansour, and Robert~E. Schapire.
\newblock Why averaging classifiers can protect against overfitting.
\newblock In \emph{Proceedings of the 8th International Workshop on Artificial
  Intelligence and Statistics}, 2001.

\bibitem[Greenlaw et~al.(1995)Greenlaw, Hoover, and
  Ruzzo]{greenlaw_limits_1995}
Raymond Greenlaw, H.~James Hoover, and Walter~L. Ruzzo.
\newblock \emph{Limits to parallel computation: P-completeness theory}.
\newblock Oxford University Press, Inc., 1995.

\bibitem[Hanneke(2016)]{hanneke2016optimal}
Steve Hanneke.
\newblock The optimal sample complexity of {PAC} learning.
\newblock \emph{Journal of Machine Learning Research}, 17\penalty0
  (38):\penalty0 1--15, 2016.

\bibitem[Johnson and Lindenstrauss(1984)]{johnson1984extensions}
William~B. Johnson and Joram Lindenstrauss.
\newblock Extensions of lipschitz mappings into a hilbert space.
\newblock \emph{Contemporary mathematics}, 26\penalty0 (189-206):\penalty0 1,
  1984.

\bibitem[Kay and Womble(1971)]{kay1971axiomatic}
David Kay and Eugene~W. Womble.
\newblock Axiomatic convexity theory and relationships between the
  {C}arath{\'e}odory, {H}elly, and {R}adon numbers.
\newblock \emph{Pacific Journal of Mathematics}, 38\penalty0 (2):\penalty0
  471--485, 1971.

\bibitem[Kruskal et~al.(1990)Kruskal, Rudolph, and Snir]{kruskal1990complexity}
Clyde~P. Kruskal, Larry Rudolph, and Marc Snir.
\newblock A complexity theory of efficient parallel algorithms.
\newblock \emph{Theoretical Computer Science}, 71\penalty0 (1):\penalty0
  95--132, 1990.

\bibitem[Kumar et~al.(1994)Kumar, Grama, Gupta, and
  Karypis]{kumar1994introduction}
Vipin Kumar, Ananth Grama, Anshul Gupta, and George Karypis.
\newblock \emph{Introduction to parallel computing: design and analysis of
  algorithms}.
\newblock Benjamin-Cummings Publishing Co., Inc., 1994.

\bibitem[Lichman(2013)]{Lichman:2013}
Moshe Lichman.
\newblock {UCI} machine learning repository, 2013.
\newblock URL \url{http://archive.ics.uci.edu/ml}.

\bibitem[Lin et~al.(2017)Lin, Guo, and Zhou]{lin2017distributed}
Shao-Bo Lin, Xin Guo, and Ding-Xuan Zhou.
\newblock Distributed learning with regularized least squares.
\newblock \emph{Journal of Machine Learning Research}, 18\penalty0
  (92):\penalty0 1--31, 2017.
\newblock URL \url{http://jmlr.org/papers/v18/15-586.html}.

\bibitem[Long and Servedio(2013)]{long_algorithms_2013}
Philip~M. Long and Rocco~A. Servedio.
\newblock Algorithms and hardness results for parallel large margin learning.
\newblock \emph{Journal of Machine Learning Research}, 14:\penalty0 3105--3128,
  2013.

\bibitem[Ma et~al.(2017)Ma, Konečný, Jaggi, Smith, Jordan, Richtárik, and
  Takáč]{ma_distributed_2017}
Chenxin Ma, Jakub Konečný, Martin Jaggi, Virginia Smith, Michael~I. Jordan,
  Peter Richtárik, and Martin Takáč.
\newblock Distributed optimization with arbitrary local solvers.
\newblock \emph{Optimization Methods and Software}, 32\penalty0 (4):\penalty0
  813--848, 2017.

\bibitem[Mcdonald et~al.(2009)Mcdonald, Mohri, Silberman, Walker, and
  Mann]{mcdonald2009efficient}
Ryan Mcdonald, Mehryar Mohri, Nathan Silberman, Dan Walker, and Gideon~S. Mann.
\newblock Efficient large-scale distributed training of conditional maximum
  entropy models.
\newblock In \emph{Advances in Neural Information Processing Systems}, pages
  1231--1239, 2009.

\bibitem[McMahan et~al.(2017)McMahan, Moore, Ramage, Hampson, and
  y~Arcas]{mcmahan2017communication}
Brendan McMahan, Eider Moore, Daniel Ramage, Seth Hampson, and Blaise~Aguera
  y~Arcas.
\newblock Communication-efficient learning of deep networks from decentralized
  data.
\newblock In \emph{Artificial Intelligence and Statistics}, pages 1273--1282,
  2017.

\bibitem[Meng et~al.(2016)Meng, Bradley, Yavuz, Sparks, Venkataraman, Liu,
  Freeman, Tsai, Amde, Owen, Xin, Xin, Franklin, Zadeh, Zaharia, and
  Talwalkar]{meng2016mllib}
Xiangrui Meng, Joseph Bradley, Burak Yavuz, Evan Sparks, Shivaram Venkataraman,
  Davies Liu, Jeremy Freeman, DB~Tsai, Manish Amde, Sean Owen, Doris Xin,
  Reynold Xin, Michael~J. Franklin, Reza Zadeh, Matei Zaharia, and Ameet
  Talwalkar.
\newblock Mllib: Machine learning in apache spark.
\newblock \emph{Journal of Machine Learning Research}, 17\penalty0
  (34):\penalty0 1--7, 2016.

\bibitem[Moler(1986)]{moler1986matrix}
Cleve Moler.
\newblock Matrix computation on distributed memory multiprocessors.
\newblock \emph{Hypercube Multiprocessors}, 86\penalty0 (181-195):\penalty0 31,
  1986.

\bibitem[Nouretdinov et~al.(2011)Nouretdinov, Costafreda, Gammerman,
  Chervonenkis, Vovk, Vapnik, and Fu]{nouretdinov2011machine}
Ilia Nouretdinov, Sergi~G. Costafreda, Alexander Gammerman, Alexey
  Chervonenkis, Vladimir Vovk, Vladimir Vapnik, and Cynthia~H.Y. Fu.
\newblock Machine learning classification with confidence: application of
  transductive conformal predictors to {MRI}-based diagnostic and prognostic
  markers in depression.
\newblock \emph{Neuroimage}, 56\penalty0 (2):\penalty0 809--813, 2011.

\bibitem[Oglic and G{\"a}rtner(2017)]{oglic2017nystrom}
Dino Oglic and Thomas G{\"a}rtner.
\newblock {N}ystr{\"o}m method with kernel k-means++ samples as landmarks.
\newblock In \emph{Proceedings of the 34th International Conference on Machine
  Learning}, pages 2652--2660, 06--11 Aug 2017.

\bibitem[Pedregosa et~al.(2011)Pedregosa, Varoquaux, Gramfort, Michel, Thirion,
  Grisel, Blondel, Prettenhofer, Weiss, Dubourg, Vanderplas, Passos,
  Cournapeau, Brucher, Perrot, and Duchesnay]{scikit-learn}
Fabian Pedregosa, Ga{\"e}l Varoquaux, Alexandre Gramfort, Vincent Michel,
  Bertrand Thirion, Olivier Grisel, Mathieu Blondel, Peter Prettenhofer, RRon
  Weiss, Vincent Dubourg, Jake Vanderplas, AAlexandre Passos, David Cournapeau,
  Matthieu Brucher, Matthieu Perrot, and Édouard~Duchesnay Duchesnay.
\newblock Scikit-learn: Machine learning in {P}ython.
\newblock \emph{Journal of Machine Learning Research}, 12:\penalty0 2825--2830,
  2011.

\bibitem[Radon(1921)]{radon1921mengen}
Johann Radon.
\newblock Mengen konvexer {K}{\"o}rper, die einen gemeinsamen {P}unkt
  enthalten.
\newblock \emph{Mathematische Annalen}, 83\penalty0 (1):\penalty0 113--115,
  1921.

\bibitem[Rahimi and Recht(2007)]{rahimi2007random}
Ali Rahimi and Benjamin Recht.
\newblock Random features for large-scale kernel machines.
\newblock In \emph{Advances in Neural Information Processing Systems}, pages
  1177--1184, 2007.

\bibitem[Rosenblatt and Nadler(2016)]{rosenblatt2016optimality}
Jonathan~D. Rosenblatt and Boaz Nadler.
\newblock On the optimality of averaging in distributed statistical learning.
\newblock \emph{Information and Inference}, 5\penalty0 (4):\penalty0 379--404,
  2016.

\bibitem[Rubinov(2013)]{rubinov2013abstract}
Alexander~M. Rubinov.
\newblock \emph{Abstract convexity and global optimization}, volume~44.
\newblock Springer Science \& Business Media, 2013.

\bibitem[Shamir and Srebro(2014)]{shamir2014distributed}
Ohad Shamir and Nathan Srebro.
\newblock Distributed stochastic optimization and learning.
\newblock In \emph{Proceedings of the 52nd Annual Allerton Conference on
  Communication, Control, and Computing}, pages 850--857, 2014.

\bibitem[Shamir et~al.(2014)Shamir, Srebro, and Zhang]{shamir2014communication}
Ohad Shamir, Nati Srebro, and Tong Zhang.
\newblock Communication-efficient distributed optimization using an approximate
  newton-type method.
\newblock In \emph{International conference on machine learning}, pages
  1000--1008, 2014.

\bibitem[Sommer and Paxson(2010)]{sommer2010outside}
Robin Sommer and Vern Paxson.
\newblock Outside the closed world: On using machine learning for network
  intrusion detection.
\newblock In \emph{Symposium on Security and Privacy}, pages 305--316, 2010.

\bibitem[Sra et~al.(2012)Sra, Nowozin, and Wright]{sra2012optimization}
Suvrit Sra, Sebastian Nowozin, and Stephen~J. Wright.
\newblock \emph{Optimization for machine learning}.
\newblock MIT Press, 2012.

\bibitem[Tukey(1975)]{tukey1975mathematics}
John~W Tukey.
\newblock Mathematics and the picturing of data.
\newblock In \emph{Proceedings of the International Congress of
  Mathematicians}, volume~2, pages 523--531, 1975.

\bibitem[Valiant(1984)]{valiant_theory_1984}
Leslie~G. Valiant.
\newblock A theory of the learnable.
\newblock \emph{Communications of the ACM}, 27\penalty0 (11):\penalty0
  1134--1142, 1984.

\bibitem[Vanschoren et~al.(2013)Vanschoren, van Rijn, Bischl, and
  Torgo]{OpenML2013}
Joaquin Vanschoren, Jan~N. van Rijn, Bernd Bischl, and Luis Torgo.
\newblock {OpenML}: Networked science in machine learning.
\newblock \emph{SIGKDD Explorations}, 15\penalty0 (2):\penalty0 49--60, 2013.

\bibitem[Vapnik and Chervonenkis(1971)]{vapnik1971uniform}
Vladimir~N. Vapnik and Alexey~Y. Chervonenkis.
\newblock On the uniform convergence of relative frequencies of events to their
  probabilities.
\newblock \emph{Theory of Probability \& Its Applications}, 16\penalty0
  (2):\penalty0 264--280, 1971.

\bibitem[Vitter and Lin(1992)]{vitter1992learning}
Jeffrey~S. Vitter and Jyh-Han Lin.
\newblock Learning in parallel.
\newblock \emph{Information and Computation}, 96\penalty0 (2):\penalty0
  179--202, 1992.

\bibitem[Von~Luxburg and Sch{\"o}lkopf(2011)]{luxburg2009statistical}
Ulrike Von~Luxburg and Bernhard Sch{\"o}lkopf.
\newblock Statistical learning theory: models, concepts, and results.
\newblock In \emph{Inductive Logic}, volume~10 of \emph{Handbook of the History
  of Logic}, pages 651--706. Elsevier, 2011.

\bibitem[Witten et~al.(2017)Witten, Frank, Hall, and Pal]{witten2016data}
Ian~H. Witten, Eibe Frank, Mark~A. Hall, and Christopher~J. Pal.
\newblock \emph{Data Mining: Practical machine learning tools and techniques}.
\newblock Elsevier, 2017.

\bibitem[Zhang et~al.(2013)Zhang, Duchi, and
  Wainwright]{zhang_communication-efficient_2013}
Yuchen Zhang, John~C. Duchi, and Martin~J. Wainwright.
\newblock Communication-efficient algorithms for statistical optimization.
\newblock \emph{Journal of Machine Learning Research}, 14\penalty0
  (1):\penalty0 3321--3363, 2013.

\bibitem[Zinkevich et~al.(2010)Zinkevich, Weimer, Smola, and
  Li]{zinkevich/nips/2010}
Martin Zinkevich, Markus Weimer, Alexander~J. Smola, and Lihong Li.
\newblock Parallelized stochastic gradient descent.
\newblock In \emph{Advances in Neural Information Processing Systems}, pages
  2595--2603, 2010.

\end{thebibliography}
\end{document}